\documentclass[10pt, letterpaper, conference]{IEEEtran}
\IEEEoverridecommandlockouts
\usepackage{cite}
\usepackage{amsmath,amssymb,amsfonts}
\usepackage{graphicx}
\usepackage{textcomp}
\usepackage[table]{xcolor}
\def\BibTeX{{\rm B\kern-.05em{\sc i\kern-.025em b}\kern-.08em
    T\kern-.1667em\lower.7ex\hbox{E}\kern-.125emX}}

\usepackage{pifont}
\usepackage{tabularx}
\usepackage{amsmath, amsthm, mdframed, calc}
\usepackage{adjustbox}
\usepackage{algorithm}
\usepackage{algpseudocode}

\newtheorem{problem}{Problem}

\newtheorem{lemma}{Lemma}
\newtheorem{theorem}{Theorem}
\newtheorem{remark}{Remark}

\usepackage{subfig,float}
\usepackage{multirow}
\usepackage{makecell}
\usepackage{graphicx}
\usepackage{comment}
\usepackage{anyfontsize}
\newcommand{\tinyscript}{\fontsize{5.9}{6.7}\selectfont}
\usepackage[normalem]{ulem}
\usepackage[bookmarks=false]{hyperref}

\begin{document}

\title{Fast Convergence through Distributed Augmentation for Class-Imbalanced Federated Learning\\
\thanks{This work was supported by the research grant titled ``Robust Privacy Preserving Machine Learning Models" under the Responsible AI theme titled ``Privacy Enhancing Strategy" by INDIAai, Ministry of Electronics and Information Technology, New Delhi, India.}
}

\author{\IEEEauthorblockN{Arathi Nair M$^*$, Harshan Jagadeesh$^{*,\dagger}$ and Anwitaman Datta$^\#$}
\IEEEauthorblockA{
\textit{$^*$Amar Nath and Shashi Khosla School of Information Technology, Indian Institute of Technology Delhi, India} \\
\textit{$^\dagger$Department of Electrical Engineering, Indian Institute of Technology Delhi, India} \\
\textit{$^\#$College of Computing and Data Science, Nanyang Technological University, Singapore} \\
}
}

\maketitle
\begin{abstract}
In federated learning, mitigating class imbalance is essential to improve minority-class performance. A common approach to address this problem is to augment minority-class samples to achieve local class balance. Existing approaches treat augmentation as a heuristic and do not establish how the amount of augmentation influences the convergence of federated learning, leading to excessive augmentation and increased training time. To address this limitation, we first establish the relationship between augmentation and the convergence behavior of federated learning. Leveraging this insight, we propose DAFL, a distributed augmentation framework that determines the minimum augmentation required for each client-class pair by jointly minimizing augmentation and training time while constraining global class imbalance, thereby improving minority-class F1-score. Experimental results demonstrate that DAFL consistently improves minority-class F1-score while substantially reducing training time, particularly under severe global class imbalance and high label proportion imbalance.
\end{abstract}

\section{Introduction}
    With growing interest in training machine-learning models directly on edge devices, Federated Learning (FL) has become a promising paradigm for privacy-preserving training across distributed edge devices. However, the distributed nature of FL gives rise to statistical heterogeneity, where clients observe different data distributions due to variations in operating conditions and geographical locations. Consequently, the local datasets are often class imbalanced. For example, in healthcare monitoring, battery-powered smartwatches collect ECG signals from users with varying cardiac conditions, causing some disease classes to be rare or even absent on individual devices \cite{healthcareIoT}. Such class imbalance degrades the minority-class performance of the resulting global model \cite{low_minority_f1, class_imbalance_survey}. A common approach to mitigate class imbalance is to augment minority-class samples before local training. However, augmentation increases the computational burden of local training by generating and processing additional samples. Since many edge devices are battery-powered \cite{resource_constraint, resource_constraint_2}, this additional computation increases energy consumption and prolongs the training process. Therefore, it is important to design FL strategies that improve minority-class performance while minimizing overall training time.

    \subsection{Motivation}
    Existing approaches to class imbalance in FL can be broadly categorized into non-augmentation-based and augmentation-based methods. Non-augmentation-based methods improve model learning without generating additional samples. These methods explicitly address class imbalance through techniques such as class reweighting and aggregation strategies \cite{nonAug1_loss4CI,nonAug7_loss4CI3, nonAug5_aggregation4CI1, nonAug6_aggregation4CI2} or tackle more general forms of statistical heterogeneity using loss function modifications and aggregation schemes \cite{nonAug3_loss4nonIID1,nonAug4_loss4nonIID2,nonAug11_loss4nonIID3,nonAug13_loss4nonIID5,nonAug9_aggregation4nonIID1,nonAug10_aggregation4nonIID2,nonAug8_aggregation4nonIID3}. 
    While the former are primarily designed to improve minority-class performance, the latter seek to reduce the adverse effects of client drift and accelerate convergence under non-IID data distributions, thereby providing only indirect benefits to minority classes. Furthermore, since these methods do not increase the diversity of minority-class samples, their ability to improve minority-class performance is limited under severe class imbalance. Although some approaches designed for statistical heterogeneity also improve convergence speed \cite{nonAug11_loss4nonIID3}, reducing the overall training time is generally not their primary design objective.
    
    Augmentation-based techniques \cite{Aug2_GAN2, Aug5_SMOTE2, Aug7_cutmix, Aug8_randaugment, Aug10_diffusion} alleviate class imbalance by generating synthetic minority-class samples, thereby improving minority-class performance under severe imbalance. However, this improvement comes at the expense of additional computational overhead during both data generation and local model training \cite{aug_cost}. Therefore, these methods are less suitable where reducing training time is critical. Interestingly, \cite{Aug3_arpit} shows that achieving a desired global class balance is sufficient to improve the minority-class F1-score, suggesting that balancing every client's local dataset may be unnecessary. However, \cite{Aug3_arpit} primarily focuses on privacy-preserving augmentation using homomorphic encryption, introducing substantial computational overhead.
    
    Motivated by the above discussion, we investigate augmentation from a different perspective. Rather than asking how much augmentation is needed to maximize minority-class performance, we ask whether augmentation can be optimally allocated across clients to jointly improve minority-class performance and reduce the overall training time. Specifically, we seek to answer the following questions: (i) Can augmentation be designed to simultaneously improve convergence and minority-class performance in federated learning? (ii) If so, how should augmentation be distributed across clients so that the federation achieves faster convergence while accounting for a given global class balance? 
    \begin{table}[t]
    \caption{
    Comparison of existing methods with DAFL. \ding{52}, \ding{56}, and \ding{108} denote yes, no, and indirect support, respectively. Indirect: addressed by some works or achieved as a byproduct. CI: class imbalance; MS: minority samples; CD: client drift; LA: latency-aware; Limit.: severe classimbalance limitation.
    }
    \label{tab:comparison_novelty}
    \centering
    \tinyscript
    \setlength{\tabcolsep}{2pt}
    \renewcommand{\arraystretch}{1.1}
    \begin{tabular}
    {|p{1.25cm}|p{2.3cm}|c|c|c|c|p{2.8cm}|}
    \hline
    \multicolumn{2}{|l|}{\textbf{Category / References}} & \textbf{CI} & \textbf{MS} & \textbf{CD} & \textbf{LA} & \textbf{Limit.} \\
    \hline
    \multirow{2}{*}{\makecell[l]{Loss\\modification}}
    & Imbalance \cite{nonAug1_loss4CI,nonAug7_loss4CI3}
    & \ding{52} & \ding{56} & \ding{108} & \ding{56}
    & \multirow{4}{2.3cm}{\raggedright Does not increase diversity of minority-class data} \\
    \cline{2-6}
    & Non-IID \cite{nonAug3_loss4nonIID1,nonAug4_loss4nonIID2,nonAug11_loss4nonIID3,nonAug13_loss4nonIID5}
    & \ding{108} & \ding{56} & \ding{52} & \ding{56} & \\
    \cline{1-6}
    \multirow{2}{*}{\makecell[l]{Aggregation-\\based}}
    & Imbalance \cite{nonAug5_aggregation4CI1,nonAug6_aggregation4CI2}
    & \ding{52} & \ding{56} & \ding{108} & \ding{56} & \\
    \cline{2-6}
    & Non-IID \cite{nonAug9_aggregation4nonIID1,nonAug10_aggregation4nonIID2,nonAug8_aggregation4nonIID3}
    & \ding{108} & \ding{56} & \ding{52} & \ding{56} & \\
    \hline
    \multicolumn{2}{|p{3.3cm}|}{Data augmentation \cite{Aug2_GAN2,Aug5_SMOTE2,Aug7_cutmix,Aug8_randaugment,Aug10_diffusion,Aug3_arpit}}
    & \ding{52} & \ding{52} & \ding{108} & \ding{56}
    & Heuristic augmentation policies \\
    \hline
    \rowcolor[HTML]{34FF34}
    \multicolumn{2}{|l|}{\textbf{DAFL (proposed)}}
    & \ding{52} & \ding{52} & \ding{52} & \ding{52} & \\
    \hline
    \end{tabular}
    \end{table}
    
    \subsection{Contributions}
    To address the above-mentioned challenges, we establish a relationship between data augmentation and the convergence behavior of FL by analyzing the mismatch between local and global class distributions across participating clients. Our analysis shows that reducing label proportion imbalance through augmentation accelerates convergence, thereby providing a principled foundation for augmentation-aware federated learning. Guided by the theoretical analysis, we formulate distributed augmentation as a latency-aware optimization problem that determines the minimum class-specific augmentation required across participating clients to achieve a desired global class distribution while minimizing augmentation overhead. Subsequently, we present DAFL, a server-coordinated distributed augmentation framework that computes client-specific augmentation assignments. By reducing both label proportion imbalance and unnecessary augmentation, DAFL accelerates convergence while improving minority-class performance. Extensive experiments on the CIFAR-10 dataset under varying degrees of class imbalance demonstrate that DAFL consistently outperforms existing baselines by achieving higher minority-class F1-scores while requiring less augmentation and shorter training time to attain the target performance.

\section{System Model for Class-Imbalanced FL}
    We consider a $C$-class classification task for some $C\in\mathbb{N}$ in a supervised FL setup with a central server and $K$ resource-constrained clients. Each client $k \in \{1,\ldots, K\}$ has a local dataset $D_k = \bigcup_{c=1}^{C}D_k^{(c)}$, where $D_k^{\left(c\right)}$  contains samples of class $c$. We denote the number of samples at client $k$ by $n_k = \sum_c n_k^{\left(c\right)}$, where $n_k^{\left(c\right)}$ is the number of samples in class $c \in \{1,\ldots,C\}$ at client $k$. 
    Local class count, denoted by $\mathbf{n}_k$, is defined as the vector of class-wise sample counts at client $k$, such that $\mathbf{n}_k=[n_k^{(c)}]_{1\times C}\in\mathbb{Z}_{\ge0}^{1\times C}$. Collecting these vectors for all clients gives the client-class count matrix, denoted by $\mathbf{N}=[n_k^{(c)}]_{K\times C}\in\mathbb{Z}_{\ge0}^{K\times C}$. 
    Conceptually, the global dataset is the union of all local datasets and has $n = \sum_k n_k$ samples. Global class count, denoted by $\mathbf{\bar n}$, is defined as the vector of class-wise sample counts in the global dataset, such that $\mathbf{\bar n} = [n^{(c)}]_{1\times C}$, where $n^{\left(c\right)} = \sum_{k=1}^K n_k^{\left(c\right)}$ is the total number of samples belonging to class $c$. We define the fraction of samples per client as 
    \begin{equation} \label{eqn:pk}
        {\small p_k = \frac{n_k}{n}}.
    \end{equation} 
    For each client $k$, we define the local class proportion of class $c$, denoted by $\pi_k(c)$, as the fraction of samples belonging to class $c$ at client $ k$, given by
    \begin{equation} \label{eqn:pikc}
        {\small \pi_k(c)=\frac{n_{k}^{(c)}}{n_k}}.
    \end{equation}
    Then the local class distribution of client $k$ is  
    $[\pi_k{(c)}]_{1\times C} \in \mathbb{R}_{\ge 0}^{1 \times C}$.
    Similarly, for the global dataset, we define the global class proportion of class $c$, denoted by $\pi(c)$, as the fraction of samples belonging to class $c$, given by
    \begin{equation} \label{eqn:pic}
        {\small \pi(c) = {\textstyle\sum_k p_k \pi_k(c)} = \frac{n^{(c)}}{n}}.
    \end{equation}
    Then the global distribution is 
    $[\pi{(c)}]_{1\times C} \in \mathbb{R}_{\ge 0}^{1 \times C}$.
    
    Having established the data distribution across clients, we now describe the FL training setup. In this setup, clients communicate only with the central server and not with each other. The objective of FL is to learn the optimal global model parameter, $\mathbf{w}^* \in \mathbb{R}^d$, such that $d$ denotes the dimension of the model parameter vector, by minimizing the global loss function given by
    $F\!\left(\mathbf{w}\right)\! =\! \frac{1}{n} \sum_{k=1}^K \sum_{\xi \in D_k} f(\mathbf{w};\xi)$,
    where $f(\mathbf{w};\xi)$ denotes the loss of the model parameter $\mathbf{w}$ for the training sample $\xi$ from client $k$. However, $\mathbf{w}^*$ cannot be computed directly due to the decentralized nature of the data. 
    Instead, FL follows an iterative procedure to update the global model. At round $t \ge 0$, the server broadcasts the global model $\mathbf{w}^{(t)}$ to all clients. Subsequently, each client uses $\mathbf{w}^{(t)}$ to minimize $f(\mathbf{w};\xi)$, where $\xi$ denotes a training sample drawn uniformly at random from the local dataset $D_k$, using Stochastic Gradient Descent (SGD) such that $\mathbb{E}_{\xi\sim D_k} f(\mathbf{w};\xi) = F_k\left(\mathbf{w}\right)$.
    After $E \in \mathbb{N}$ local updates with a learning rate, $\eta$, the local model denoted by
        $\mathbf{w}_k^{\left(t,E\right)} = \mathbf{w}^{\left(t\right)}-\eta \mathbf{g}_k^{\left(t\right)}$ where $\mathbf{g}_k^{\left(t\right)} = \sum_{i=0}^{E-1}\mathbf{g}_k^{\left(t,i\right)}$
    is obtained where $\mathbb E[\mathbf g_k^{(t,i)} \!\mid\! \mathbf w^{(t,i-1)}]\! =\! \mathbb E[\nabla F_k(w_k^{(t,i)})\mid \mathbf w^{(t,i-1)}]$.
    Subsequently, the clients only share the local model updates with the server for aggregation to obtain the global model parameter as
    \begin{IEEEeqnarray}{lrcl} 
        &\mathbf{w}^{\left(t+1\right)} 
        &=& \textstyle\sum_{k=1}^K p_k\mathbf{w}_k^{\left(t,E\right)}
        \label{eqn:E_local_update}
        = \mathbf{w}^{\left(t\right)} - \eta \mathbf{g}^{\left(t\right)},
    \end{IEEEeqnarray}
    where
    \begin{IEEEeqnarray}{rcl}
    \label{eqn:aggregated_stochastic_gradient}
        \mathbf{g}^{\left(t\right)} &=& \textstyle\sum_{k=1}^K \textstyle\sum_{i=0}^{E-1} p_k\mathbf{g}_k^{\left(t,i\right)} \qquad \qquad
    \end{IEEEeqnarray}
    is the aggregated stochastic gradient.
    This iterative process continues until the global model reaches a stationary point. To study the factors affecting convergence and to understand the effect of imbalance, we present a formal theoretical convergence analysis in the next section.
    
\section{Convergence Analysis}
    To analyze the impact of class imbalance on convergence, we decompose the local loss function into class-wise components. Specifically, we define the local loss function conditioned on class $c$ at client $k$ as 
    \begin{IEEEeqnarray}{rcl} \label{eqn:local_class_conditional_obj}
        F_{k,c}\left(\mathbf{w}\right)
        &=&
        \mathbb{E}_{\xi\sim D_k^{(c)}}[f(\mathbf{w};\xi)] = \frac{1}{n_k^{(c)}}\sum_{\xi \in D_k^{(c)}} f(\mathbf{w};\xi).
    \end{IEEEeqnarray}
    Here, $F_{k,c}(\mathbf{w})$ is the loss contributed by samples of class $c$ at client $k$. Using \eqref{eqn:pikc}, we write the local loss function as
    \begin{IEEEeqnarray}{rcl} \label{eqn:final_local_obj}
        F_k(\mathbf{w})
        &=&
        \mathbb{E}_{\xi\sim D_k} f(\mathbf{w};\xi)
        =
        \textstyle\sum_{c=1}^{C}\pi_k(c)F_{k,c}(\mathbf{w}).
    \end{IEEEeqnarray}
    
    Similarly, we define the global class-conditional loss function for class $c$, denoted by $\bar{F}_c(\mathbf{w})$, by aggregating the corresponding class-conditional losses across clients. Since client $k$ contributes $p_k\pi_k(c)$ fraction of class $c$ samples to the global dataset, the normalized contribution of client $k$ among all class $c$ samples is given by $\frac{p_k\pi_k(c)}{\pi(c)}$. Hence,
    \begin{equation} \label{eqn:class_conditional_obj}
        \bar{F}_c(\mathbf{w})
        =
        \sum_{k=1}^{K}
        \frac{p_k\pi_k(c)}{\pi(c)}
        F_{k,c}(\mathbf{w}).
    \end{equation}
    
    \noindent Using \eqref{eqn:pic}, we express the global loss function as
    \begin{equation} \label{eqn:final_global_obj}
        F\left(\mathbf{w}\right) = \textstyle\sum_{k=1}^K p_kF_k(\mathbf{w}) = \textstyle\sum_{c=1}^C \pi\left(c\right) \bar F_{c}\left(\mathbf{w}\right).
    \end{equation} 
    
    For the convergence analysis, let $\mathcal F_{t,i}$ denote the filtration generated by all randomness up to update $i$ of round $t$. Further, using the definitions in \eqref{eqn:local_class_conditional_obj}, \eqref{eqn:final_local_obj}, \eqref{eqn:class_conditional_obj} and \eqref{eqn:final_global_obj}, we make the following assumptions, and then present our main results.
    \begin{itemize}
        \item \textbf{A1 (L-smoothness)}:
    Each $F_k(\cdot)$ is $L$-smooth, i.e., for any $\mathbf{x}$ and $\mathbf{y}$,
    $\|{\nabla F_k(\mathbf x) - \nabla F_k(\mathbf y)}\| \le L\|{\mathbf x-\mathbf y}\|$.
        
    \item\textbf{A2 (Conditional unbiased stochastic gradients)}
    For all round $t$ and update $i$, $\mathbb E[\mathbf g_k^{(t,i)}\! \mid \!\mathcal F_{t,i}]\! =\! \nabla F_k(\mathbf w_k^{(t,i)}), \forall k$.
    
    \item\textbf{A3 (Bounded class gradients)}
    There exists a constant $G>0$ such that
        $\|\nabla \bar{F}_c\left(\mathbf{w}\right)\|\le G,
        \forall c,\mathbf{w}$.
    
    \item\textbf{A4 (Bounded conditional variance and conditional independence)}
    For every client $k$, round $t$, and update $i$, $\mathbb E[\|{\mathbf g_k^{(t,i)}\! -\! \nabla\! F_k(\mathbf w_k^{(t,i)})}\|^2\! \mid \!\mathcal F_{t,i}] \!\le\! \sigma^2$ and, conditioned on $\mathcal F_{t,i}$, the random vectors $\{\mathbf g_k^{(t,i)}\!-\!\nabla F_k(\mathbf w_k^{(t,i)})\}_{k=1}^K$ are mutually independent across clients.
    \item\textbf{A5 (Bounded Conditional Gradient Heterogeneity)}
        $\sup_{\mathbf w} \|{\sum_c \pi_k(c)(\nabla F_{k,c}(\mathbf w)-\nabla\bar F_c(\mathbf w))}\|^2 \;\le\; \Gamma_k^2, \ \forall k$. 
    \end{itemize}
    \begin{lemma}
    \label{lem:gradient_decomposition}
    
    Suppose \textbf{A3} and \textbf{A5} hold. Then for every client $k$
    and every model parameter $\mathbf w$,
    \begin{IEEEeqnarray}{rcl}
        \|\nabla F_k(\mathbf w)\|^2 &\le&2\|\nabla F(\mathbf w)\|^2 + 4CG^2\delta_k^2 + 4\Gamma_k^2,
        \label{eqn:lemma_gradient_bound}
        \end{IEEEeqnarray}
    where
    \vspace{-1mm}
    \begin{IEEEeqnarray}{rcl}
                \label{eqn:delta_k2}
                    \delta_k^2 &=& \left\| \pi_k-\pi\right\|_2^2.
    \end{IEEEeqnarray}
    \end{lemma}
    \begin{proof}
        See Appendix \ref{proof:gradient_decomposition} for the proof.
    \end{proof}
    Lemma~\ref{lem:gradient_decomposition} decomposes the discrepancy between local and global gradients into two components: (i) client-level label proportion imbalance, quantified by $\delta_k^2$, and (ii) client-level conditional gradient shift, quantified by $\Gamma_k^2$. Substituting this decomposition into the standard FedAvg descent analysis yields the following convergence bound.
    \begin{theorem} 
    \label{thm:convergence}
        Suppose \textbf{A1}--\textbf{A5} hold. If $T \ge 100L^2E$, where $T$ denotes the total number of communication rounds and $\eta = \frac {1} {\sqrt{ET}}$, then FedAvg 
        in \eqref{eqn:aggregated_stochastic_gradient} 
        with full participation, fixed data size and $E$ local updates satisfies
         \\
        \begin{small}
            \begin{IEEEeqnarray}{lrcl}
                \frac{1}{T}\! \textstyle\sum_{t=0}^{T-1}\!\mathbb{E} &[ \|\nabla F(&\mathbf{w}^t&)\|^2 ] \;=\; O\!\left(\frac{F(\mathbf{w}^0)-F(\mathbf w^*)}{\sqrt{ET}}\right)\! +\! O\!\left(\frac{L\sigma_g^2}{\sqrt{ET}}\right) \nonumber\\ 
                &&+& O\!\left(\frac{\left(E-1)^2L^2(CG^2\mathcal I+\Gamma+\sigma^2\right)}{T}\right),
                \label{eqn:thm_convergence_rate}
                \end{IEEEeqnarray}
        \end{small}
                where,
                \begin{small}
                \begin{IEEEeqnarray}{rcl}
                \label{eqn:sigma_g2}
                    \sigma_g^2 &=& \sigma^2\textstyle\sum_k p_k^2,\\
                \label{eqn:label_imbalance_measure}
                    \mathcal{I}&=&\textstyle\sum_k p_k\delta_k^2,\\
                \label{eqn:conditional_heterogeneity_measure}
                    \Gamma&=&\textstyle\sum_k p_k\Gamma_k^2.
            \end{IEEEeqnarray}
                \end{small}
        
    \end{theorem}
    \begin{proof}
        See Appendix \ref{proof:thm_convergence} for a sketch of the proof. Detailed proof of this result is omitted due to lack of space.
    \end{proof}

    \begin{remark}
    Existing analyses, such as \cite{nonAug11_loss4nonIID3, nonAug13_loss4nonIID5}, model data heterogeneity using a single bounded-dissimilarity constant. In contrast, we decompose heterogeneity into a label proportion imbalance term ($\mathcal I$) and a conditional gradient heterogeneity term ($\Gamma$), thereby explicitly quantifying the contribution of each term to the convergence bound. Since $\mathcal I$ depends only on class proportions, it can be directly reduced by augmenting under-represented classes locally, whereas $\Gamma$ cannot.
    \end{remark}
    
    In Theorem \ref{thm:convergence}, the terms involving $\mathcal I$ and $\Gamma$ characterize the impact of label proportion imbalance and conditional gradient heterogeneity, respectively, on the convergence time.
    Hence, we identify label proportion imbalance and conditional gradient heterogeneity as two sources of client drift in FL. We note that both label proportion imbalance and conditional gradient heterogeneity depend on the local and global distributions. However, unlike label proportion imbalance, conditional gradient heterogeneity is influenced by the quality of the augmented samples, which depends on factors such as the augmentation technique used. 
    Consequently, it is challenging to bring the conditional gradient heterogeneity term as a parameter for optimization  via the augmentation framework. This motivates us to focus on label proportion imbalance, which forms the basis of the optimization problem developed subsequently. 
    
    \subsection{Label Proportion Imbalance Measure, $\mathcal{I}$}
    From Theorem \ref{thm:convergence}, it is clear that reducing $\mathcal{I}$ accelerates convergence by mitigating client drift. In particular, we can force $\mathcal{I}\!=\!0$ by ensuring that every client has the global class proportion, i.e., $\pi_k\!=\!\pi,\forall k$. 
    
    \subsection{Global Class Imbalance Measure, $J$}
    When class balance holds, the proportion of each class in a dataset follows a uniform distribution.
    Hence, local class imbalance exists when $\pi_k(c) \not= \frac{1}{C}$, for any $k,c$, and global class imbalance exists when $\pi(c) \not= \frac{1}{C}$, for any $c$. To quantify class imbalance, we use a quadratic penalty to measure how far the distribution deviates from a uniform distribution. In particular, we quantify local class imbalance metric, and global class imbalance metric, denoted by $J_k(\mathbf N)$ and $J(\mathbf N)$, as
        \[J_k(\mathbf N)\!=\! \textstyle\sum_{c=1}^{C} \left( \pi_k(c)\! -\! \frac{1}{C} \right)^2 \text{and } J(\mathbf N) \!=\! \textstyle\sum_{c=1}^{C} \left(\pi(c)\! -\! \frac{1}{C} \right)^2,\]
    \noindent where $\pi_k(c)$ and $\pi(c)$ are computed from $\mathbf N$ using \eqref{eqn:pikc} and \eqref{eqn:pic} respectively. Note that non-uniform class distribution results in $J>0$ and may lead to poor performance of minority classes. 
    
    We note that $\mathcal{I}$ and $J$ are different properties. $\mathcal{I}\!=\!0$ requires $\pi_k\!=\! \pi$ while $J=0$ requires uniform class distribution. Consequently, $\mathcal{I}=0 \nLeftrightarrow J=0$. One potential solution to mitigate both $\mathcal{I}$ and $J$ is to employ data augmentation to balance the local dataset of each client. This yields both $J=0$ and $\mathcal{I} = 0$. However, locally balancing the datasets increases training and augmentation time, making it infeasible for every client to do so. Consequently, clients should augment only a limited number of samples to jointly reduce augmentation and training times while improving minority-class performance. 
    
\section{DAFL: Distributed Augmentation for Federated Learning}
    Applying augmentation to align the local class distribution with the global class distribution reduces $\mathcal{I}$, whereas aligning the global class distribution with the uniform distribution reduces $J$. Since different client-class pairs $(k,c)$ require different amounts of augmentation, the augmentation decision must be made separately for each client and class. Accordingly, let $a_k^{(c)}$ denote the number of samples augmented for class $c$ at client $k$, such that the augmentation matrix, denoted using $\mathbf{A}$, is given by $\mathbf{A} = [a_k^{(c)}]_{K\times C}$. Subsequently, the total number of samples augmented by client $k$, denoted using $a_k$, is given by $a_k = \sum_{c=1}^{C} a_{k}^{(c)}$. Using $\mathbf{A}$, the updated client-class count matrix after augmentation, denoted by $\mathbf{N'}$, is given by $\mathbf{N'} = \mathbf{N}+\mathbf{A} = [n_k^{(c)} + a_k^{(c)}]_{K\times C}$. Consequently, $\mathbf{A}$ affects both the local and global class distributions, thereby influencing $\mathcal{I}$ and $J$.
    Among all feasible augmentation matrices, our objective is to identify one that minimizes both augmentation and training time. To characterize each client's computational capability, we introduce two delay factors that serve as relative per-sample time proxies in the optimization objective. Let $\mathbf{\zeta}$ represent the relative time cost of each client to augment one sample. Similarly, let $\mathbf{\zeta}'$ denote the relative time cost of each client to train on one sample. 
    
    \subsection{Problem Formulation}
    To formulate the optimization problem, we quantify the  augmentation time and convergence time in the presence of label proportion imbalance. We define the estimated augmentation time, denoted by $T_{a}$, as the time cost associated with the clients to complete augmentation and is given by,
        \begin{equation} \label{eqn:Taug}
        T_{a} = \zeta \max_k\textstyle\sum_c a_{k}^{(c)}.
        \end{equation}
    Since our objective is to mitigate the effect of label proportion imbalance, we isolate the component of the convergence bound that depends on $\mathcal I$. Bounding this component by $\epsilon \ll1$ gives, 
           $\frac{\left( (E-1)^2L^2CG^2\mathcal{I} \right)}{{T}}  < \epsilon$, as a condition to nullify the effect of label proportion imbalance on the convergence time. 
       As a consequence, we approximate the required communication rounds as
            $T \approx \frac{\left( 1+(E-1)^2L^2CG^2\mathcal{I} \right)}{{\epsilon}}$, where $T$ serves as a proxy for the number of communication rounds rather than the actual number of rounds required for convergence. This proxy
            captures both estimated convergence time when $\mathcal{I} = 0$ and the additional estimated time incurred due to $\mathcal{I} > 0$. 
        We define $T_{r} = \zeta' \max\limits_k (n_k+a_k)$, as the estimated time incurred by the client with the most samples to complete one training round. Then the total estimated training time, denoted by $T_{t}$, is 
        \begin{equation} \label{eqn:Ttrain}
                T_{t} =T T_{r} = \frac{\left( 1+(E-1)^2L^2CG^2\mathcal{I} \right)}{{\epsilon}}  \zeta' \max_k(n_k+a_k).
        \end{equation}

    The above definitions lead to a tradeoff between $T_a$ and $T_t$. While increasing the number of augmented samples may reduce $T_t$ by lowering $\mathcal{I}$, it also increases the augmentation time $T_a$. To capture this tradeoff, we use $\lambda_1$ and $\lambda_2$ to control the importance of $T_{t}$ and $T_{a}$, respectively. In this direction, we present the optimization problem for a given $\mathbf{N}$ as
       \begin{problem} \label{prob:opt_prob}
        Solve
            \begin{equation*}
                \begin{aligned}
                    \mathbf{A}^* = \mathop{\arg\min}_{\mathbf{A}=\left[a_k^{(c)}\right]_{K\times C}} \quad & 
                    \lambda_1 T_{t} + 
                    \lambda_2 T_{a} \\
                    \text{s.t.} \quad & J(\mathbf{N}+\mathbf{A}^*) \le \tau, \quad a_k^{(c)} \in \mathbb{Z}_{\ge 0}, \forall k,c.
                \end{aligned}
            \end{equation*}
        \end{problem}
    \noindent where $\tau$ is the maximum tolerable class imbalance after augmentation. Since $\tau\!\approx\!0$ implies a nearly uniform global distribution, achieving it requires more augmented samples than a larger value of $\tau$. The choice of $\lambda_1$ and $\lambda_2$ reflects the tradeoff between training and augmentation time. When augmentation overhead is a major concern due to limited computational capability, a larger $\lambda_2$ can be selected to prioritize lower augmentation time. Additionally, since $T_{t}\! \gg\! T_{a}$, as $\epsilon\!\ll\!1$, the effect of augmentation is reflected in the objective only when $\lambda_2\!\gg\!\lambda_1$. Finally, DAFL optimizes the amount of augmentation required to improve minority-class performance while minimizing overall time, rather than the choice of augmentation method. Consequently, it is compatible with any class-preserving augmentation technique.
    \subsection{Proposed Algorithm for DAFL}\label{sec:propmeth}
    To realize $\mathbf{A}^*$, the server requires knowledge of  $\mathbf{n}_k$, which is distributed across the clients in the federation. We refer to the resulting collaborative process as distributed augmentation, in which the clients in a federation collaborate with the server to perform sufficient augmentation. Here, the clients first share their local class counts with the server to obtain $\mathbf{N}$. Given $\mathbf{N}$, the server then solves Problem \ref{prob:opt_prob} to obtain $\mathbf{A}^*$. 
    Problem \ref{prob:opt_prob} is a non-convex integer optimization problem. The integer constraints and the max operators in \eqref{eqn:Taug} and \eqref{eqn:Ttrain} make it a combinatorial search problem whose complexity grows exponentially with $K$ and $C$. Hence, we propose a greedy algorithm, polynomial in $K$ and $C$ for a fixed number of iterations, to obtain a suboptimal solution, $\hat{\mathbf{A}}$, as shown in Algorithm~\ref{alg:greedy_aug}. Distributed augmentation is performed in two phases with an augmentation step size $s$. The first phase enforces global balance, whereas the second phase performs time-aware optimization by greedily selecting, at each iteration, the client-class pair to augment by $s$ samples. A large $s$ reduces search complexity at the cost of a coarser augmentation allocation, while a small $s$ results in finer allocation at the expense of increased search complexity. In particular, since
    $\hat{\mathbf{A}}$ is built by only adding samples at each step, Phase 1 may overshoot $\tau$ when enforcing $J(\mathbf{N}+\hat{\mathbf{A}})$, and this excess cannot be corrected in Phase 2. Consequently, $\hat{\mathbf{A}}$ may include more augmented samples than required to satisfy the constraint. Once $\hat{\mathbf{A}}$ is obtained, the server returns the $k$-th row of $\hat{\mathbf{A}}$ denoted by $\hat{\mathbf{A}}_k$, to client $k$. Each client $k$ then augments its local dataset according to $\hat{\mathbf{A}}_k$. Subsequently, FL is performed on a more balanced dataset with minimal increase in training and augmentation time compared to local balancing. 
    
    \noindent\textbf{Computational Complexity:} Evaluating $J$ and $P=\lambda_1T_t+\lambda_2T_a$ requires $O(K\!C)$ time. In addition, $J$ or $P$ are computed $(KC)$ times in both phases. This results in an overall complexity of $O(KC)KC = O(K^2C^2)$ per iteration. If phases 1 and 2 require $t_1$ and $t_2$ iterations, respectively, the total complexity of DAFL is $O(t_1K^2C^2\!+\!t_2K^2C^2)\!=\!O(tK^2C^2)$,
    where $t\!=\!t_1\!+\!t_2$. As each iteration augments $s$ samples, augmenting $|\hat{\mathbf A}|$ samples requires $\frac{|\hat{\mathbf A}|}{s}$ iterations, yielding a complexity of $O(\frac{|\hat{\mathbf A}|}{s}K^2C^2)$. 

\begin{algorithm}[t]
    \caption{Greedy Approach to find $\hat{\mathbf{A}}$}
    \label{alg:greedy_aug}
    \textbf{Input:} Local class counts $\mathbf{N}_k$, threshold $\tau$, step size $s$\\
    \textbf{Output:} Augmentation matrix $\hat{\mathbf{A}} \in \mathbb{Z}_{\ge 0}^{K \times C}$
    \begin{algorithmic}[1]
    
        \State $\hat{\mathbf{A}} \leftarrow \mathbf{0}^{K \times C}$, $\mathbf{N} \leftarrow [\mathbf{N}_k]$
        
        \textbf{Phase 1: Global Balance Enforcement}
        \While{$J(\mathbf{N}+\hat{\mathbf{A}}) > \tau$}
            \State Select $(k^*,c^*)$ with the largest reduction in J on adding $s$ samples; break ties by minimizing $P = \lambda_1T_{t}+\lambda_2T_{a}$
            \State $\hat{\mathbf{A}}_{k^*,c^*} \leftarrow \hat{\mathbf{A}}_{k^*,c^*}+s$
        \EndWhile
        
        \textbf{Phase 2: Time-Aware Optimization}
        \While{feasible improvement exists}
            \State Select feasible $(k^*,c^*)$ that gives the maximum reduction in $P$ on adding $s$ samples
            \State $\hat{\mathbf{A}}_{k^*,c^*}\leftarrow \hat{\mathbf{A}}_{k^*,c^*}+s$
        \EndWhile
        
        \State \Return $\hat{\mathbf{A}}$
    \end{algorithmic}
\end{algorithm}

\begin{table}
    \caption{Dataset details for D1–D6, including client-class counts (C1–C8 denote Class 1–Class 8) and the corresponding $\mathcal{I}$ and $J$ values. D1–D3 isolate extreme $\mathcal{I}$/$J$ settings while D4–D6 cover intermediate settings. Here, the descriptors low/moderate/high are relative rather than specific thresholds.
     Also, $J \in \left[0, \frac{C-1}{C}\right]$.
    }
    \label{tab:dataset details}
    \centering
    \tinyscript
    \setlength{\tabcolsep}{1pt}
    \renewcommand{\arraystretch}{1.06}
    \begin{tabular}{|c|cccccccc|cccccccc|}
    \hline
    &\multicolumn{8}{c|}{D1} &\multicolumn{8}{c|}{D2}
    \\[-1pt]
    \hline
    $\mathcal{I}$/$J$
    &\multicolumn{8}{c|}{0.0002 / 0.4097} 
    &\multicolumn{8}{c|}{0.3982 / 0.4097}
    \\[-1pt]
    \hline
    Desc.
    &\multicolumn{8}{c|}{Low $\mathcal{I}$, High $J$} 
    &\multicolumn{8}{c|}{High $\mathcal{I}$, High $J$}\\[-1pt]
    \hline
    & C1&C2&C3&C4&C5&C6&C7&C8
    & C1&C2&C3&C4&C5&C6&C7&C8
    \\[-1pt]
    \hline
    Client 1 & 
        52&18&61&2&30&612&70&5  & 
        1&126&1&1&1&1&485&1
        \\[-1pt]
    Client 2 & 
        63&21&74&3&28&701&65&6  & 
        1&1&1&10&248&1&1&1
        \\[-1pt]
    Client 3 &      
        48&20&69&4&33&5800&60&4 &  
        395&2&1&1&1&1&1&1
        \\[-1pt]
    Client 4 &
        55&17&66&2&29&640&72&5  &
        7&22&482&1&1&1&2&33
        \\[-1pt]
    Client 5 & 
        50&22&58&3&35&590&68&6  &
        1&1&5&1&1&2&39&1
        \\[-1pt]
    Client 6 & 
        46&19&63&2&27&605&66&5  &
        1&1&1&1&1&14&1&1
        \\[-1pt]
    Client 7 &
        49&23&57&2&36&598&64&5  &
        1&1&10&4&1&4882&1&1
        \\[-1pt]
    Client 8 & 
        45&15&54&2&37&578&66&5  &
        1&1&1&1&1&2&1&2
        \\[-1pt]
    \hline
    Global & 
        408&155&502&20&255&4904&531&41 &
        408&155&502&20&255&4904&531&41
        \\[-1pt]
    \hline
    
    \hline
    &\multicolumn{8}{c|}{D3} 
    &\multicolumn{8}{c|}{D4}\\[-1pt]
    \hline
    $\mathcal{I}$/$J$
    &\multicolumn{8}{c|}{0.3494 / 0.001}
    &\multicolumn{8}{c|}{0.2187 / 0.4037} 
    \\[-1pt]
    \hline
    Desc.
    &\multicolumn{8}{c|}{High $\mathcal{I}$, Low $J$}
    &\multicolumn{8}{c|}{Moderate $\mathcal{I}$, High $J$} 
    \\[-1pt]
    \hline
    & C1&C2&C3&C4&C5&C6&C7&C8
    & C1&C2&C3&C4&C5&C6&C7&C8 \\[-1pt]
    \hline
    Client 1 & 
        1&1&1&3901&1288&1&1&2 &
        72&8&21&2&1&2&10&5 
        \\[-1pt]
    Client 2 & 
        3652&1&1&1&1&3&1&1 &
        13&1&134&3&219&1&265&6
        \\[-1pt]
    Client 3 &      
        1&1&1&2&1&2&1&1 &
        18&1&9&4&3&1680&60&4
        \\[-1pt]
    Client 4 &
        1&3664&1&1&60&4035&1&1 &
        105&7&86&2&7&40&22&5
        \\[-1pt]
    Client 5 & 
        1&1&4537&1&2959&1&4243&4144 &
        20&122&18&2&7&40&22&5 
        \\[-1pt]
    Client 6 & 
        2&1&2&1&1&1&1&1 &
        26&8&63&2&7&5&26&5 
        \\[-1pt]
    Client 7 &
        1&1&1&1&1&1&1&1 &
        9&3&17&2&6&2898&14&5 
        \\[-1pt]
    Client 8 & 
        1341&1&1&1&7&1&1&1 &
        145&5&154&2&7&78&66&5 
        \\[-1pt]
    \hline
    Global & 
        5000&3671&4545&3909&4318&4045&4250&4152 &
        408&155&502&20&255&4794&531&41 
        \\[-1pt]
    \hline
    
    \hline
    &\multicolumn{8}{c|}{D5} &\multicolumn{8}{c|}{D6}
    \\[-1pt]
    \hline
    $\mathcal{I}$/$J$
    &\multicolumn{8}{c|}{0.2989 / 0.3657}
    &\multicolumn{8}{c|}{0.5241 / 0.1309}\\[-1pt]
    \hline
    Desc.
    &\multicolumn{8}{c|}{High $\mathcal{I}$, High $J$}
    &\multicolumn{8}{c|}{High $\mathcal{I}$, Moderate $J$}\\[-1pt]
    \hline
    & C1&C2&C3&C4&C5&C6&C7&C8
    & C1&C2&C3&C4&C5&C6&C7&C8
    \\[-1pt]
    \hline
    Client 1 &  
        112&8&21&1&1&2&10&5 & 
        1&27&36&13&36&18&41&2581 \\[-1pt]
    Client 2 & 
        3&1&4&3&219&1&375&6 &
        2&47&2&8&14&60&26&1650 \\[-1pt]
    Client 3 &      
        8&1&9&4&3&80&10&2 &
        3&30&22&9&19&12&3&23 \\[-1pt]
    Client 4 &
        245&7&36&2&7&40&22&3 &
        6&5&92&41&4&21&14&4 \\[-1pt]
    Client 5 & 
        20&122&18&6&5&90&18&20 &
        9&109&90&573&78&945&52&73 \\[-1pt]
    Client 6 & 
        6&8&63&2&7&5&26&3 &
        8&47&2212&14&45&4&14&40 \\[-1pt]
    Client 7 &
        9&3&17&2&6&3898&14&3 &
        10&42&2508&38&76&32&5&97 \\[-1pt]
    Client 8 & 
        5&5&334&2&7&78&66&3 &
        1&2693&38&4&8&8&25&32 \\[-1pt]
    \hline
    Global & 
        408&155&502&22&255&4194&541&45 &
        40&3000&5000&700&280&1100&180&4500\\[-1pt]
    \hline
    \end{tabular}
\end{table}

\section{Dataset Details}
    We evaluate DAFL on an 8-class FL classification task using the CIFAR-10 dataset. We use the first eight classes, with full client participation. 
    To evaluate the robustness of DAFL, we construct six dataset configurations (D1–D6) covering different combinations of $\mathcal{I}$ and $J$, as tabulated in Table~\ref{tab:dataset details}. D1–D3 represent extreme imbalance settings, whereas D4–D6 correspond to moderate imbalance. The training set is partitioned across eight clients according to these imbalance settings, while the test set remains balanced to ensure that the results are not biased towards the majority classes. Images are normalized using a channel-wise mean and standard deviation of 0.5.

\begin{table*}[t]
\caption{Updated $\mathcal{I}$, $J$, number of augmented samples, and time associated with augmentation for LB and DAFL $(\lambda_2 = 10^q)$.}
\label{tab:updated_dataset_details}
\centering
\tinyscript
\setlength{\tabcolsep}{4pt}
\begin{tabular}{|c|c|ccc|c|ccc|c|ccc|c|ccc|ccc|}
\hline
\multirow{2}{*}{Dataset}
& \multicolumn{4}{c|}{$\mathcal{I}$}
& \multicolumn{4}{c|}{$J$}
& \multicolumn{4}{c|}{Added Samples}
& \multicolumn{4}{c|}{Time for Augmentation (s)} 
& \multicolumn{3}{c|}{Time at Server to find $\hat{\mathbf{A}}$ (s)} \\
\cline{2-20}
& LB & \multicolumn{3}{c|}{DAFL}
& LB & \multicolumn{3}{c|}{DAFL}
& LB & \multicolumn{3}{c|}{DAFL}
& LB & \multicolumn{6}{c|}{DAFL}\\
\cline{3-5}
\cline{7-9}
\cline{11-13}
\cline{15-20}
&  &  $q=2$ & $q=3$ & $q=4$
&  &  $q=2$ & $q=3$ & $q=4$
&  &  $q=2$ & $q=3$ & $q=4$ 
&  & $q=2$ & $q=3$ & $q=4$ 
&  $q=2$ & $q=3$ & $q=4$  \\
\hline
D1 & 0      & 0.0001      & 0.0001      & 0.0001      & 0.001  & 0.001 & 0.001 & 0.001 & 24913 & 24909 & 24909 & 24909 & 0.993 & 0.390 & 0.291 & 0.373 & 33.666 & 35.219 & 34.900 \\

D2 & 0.0005 & 0.1479 & 0.1479 & 0.0798 & 0.0005 & 0.001 & 0.001 & 0.001 & 35571 & 24913 & 24913 & 24911 & 7.258 & 0.351 & 0.343 & 0.364 & 34.782 & 38.081 & 40.032 \\

D3 & 0.0009 & 0.0558  & 0.2836  & 0.3492 & 0.0001 & 0.001 & 0.001 & 0.001 & 81622 & 39060 & 3080 & 7 & 5.526 & 0.682 & 0.042 & 0.297 & 55.485 & 4.238 & 0.012 \\

D4 & 0.0003  & 0.0795 & 0.0795 & 0.0461 & 0.0007 & 0.001 & 0.001 & 0.001 & 28020 & 24309 & 24309 & 24311 & 4.406 & 0.334 & 0.395 & 0.318 & 32.878 & 32.897 & 34.163 \\

D5 & 0.0005      & 0.1349 & 0.1349 & 0.0703 & 0.0005 & 0.001 & 0.001 & 0.001 & 27709 & 21014 & 21014 & 21015 & 5.462 & 0.313 & 0.286 & 0.246 & 29.523 & 28.218 & 29.037 \\

D6 & 0.0008 & 0.1651 & 0.1650 & 0.1260 & 0.0002 & 0.001 & 0.001 & 0.001 & 67434 & 18368 & 18367 & 18367 & 4.037 & 0.359 & 0.348 & 0.278 & 25.205 & 24.907 & 25.755 \\
\hline
\end{tabular}
\end{table*}

\begin{figure}[!t]
\centering
\subfloat[] {
    \includegraphics[scale = 0.235]{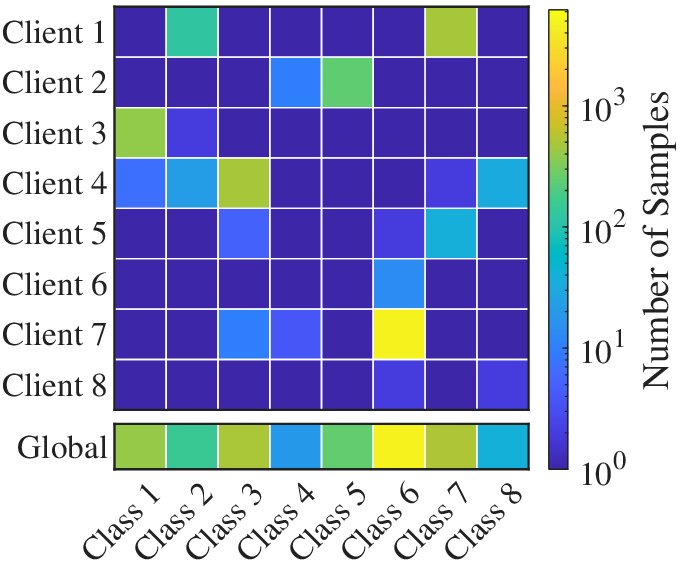}
    \label{fig:Dataset2_N}
}
\subfloat[] {
    \includegraphics[scale = 0.235]{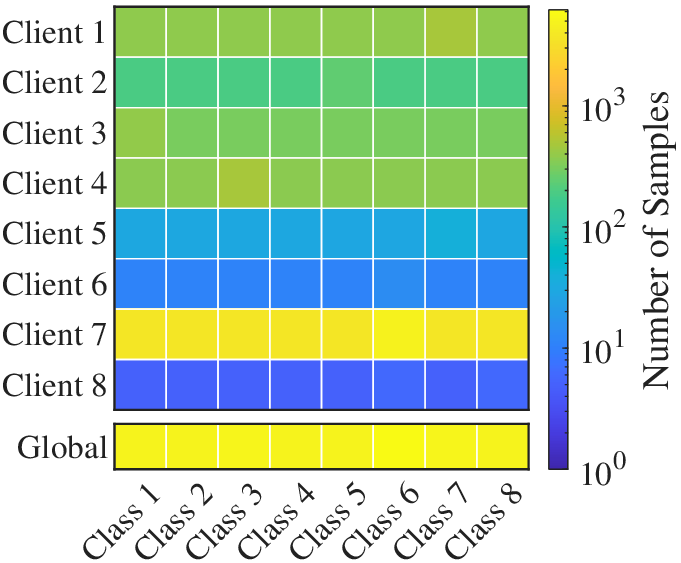}
    \label{fig:Dataset2_LBJ_Nprime}
}
\subfloat[] {
    \includegraphics[scale = 0.235]{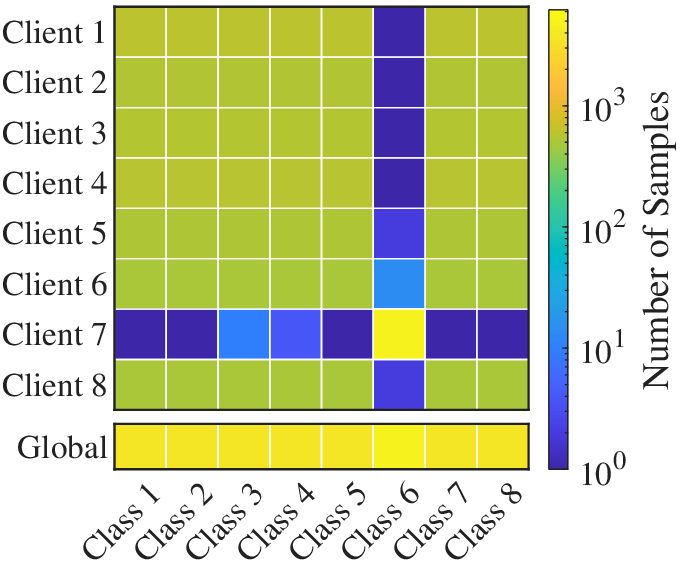}
    \label{fig:Dataset2_Nprime}
}
\caption{Heatmap for client-class counts for D2 (a) before augmentation, (b) after LB, and (c) after DAFL.}
\label{fig:updated_dataset2}
\end{figure}
\begin{figure*}[]
\centering
\vspace{-0.4cm}
\subfloat[] {
    \includegraphics[scale = 0.14]{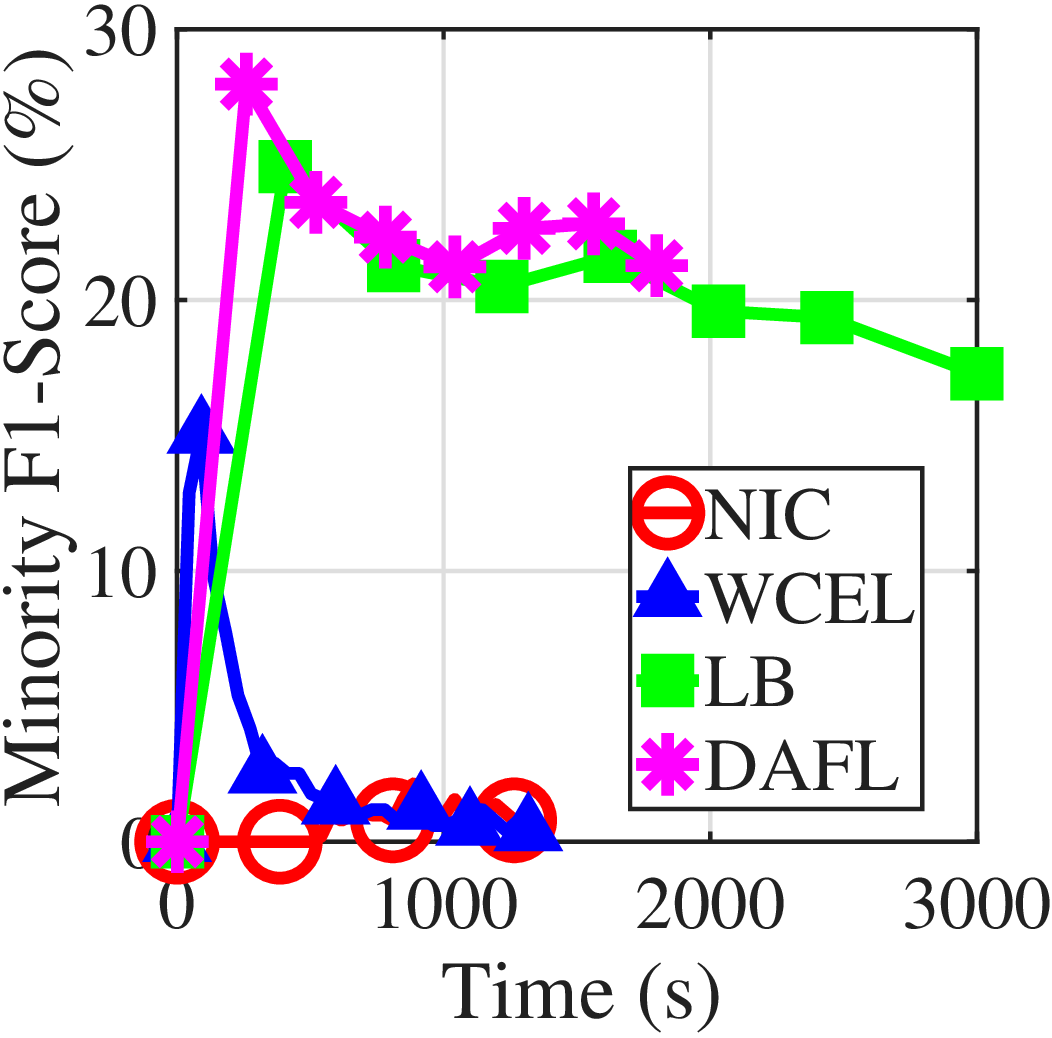}
    \label{fig:I_0_0_J_0_4097_minority_f1score}
}
\subfloat[] {
    \includegraphics[scale = 0.14]{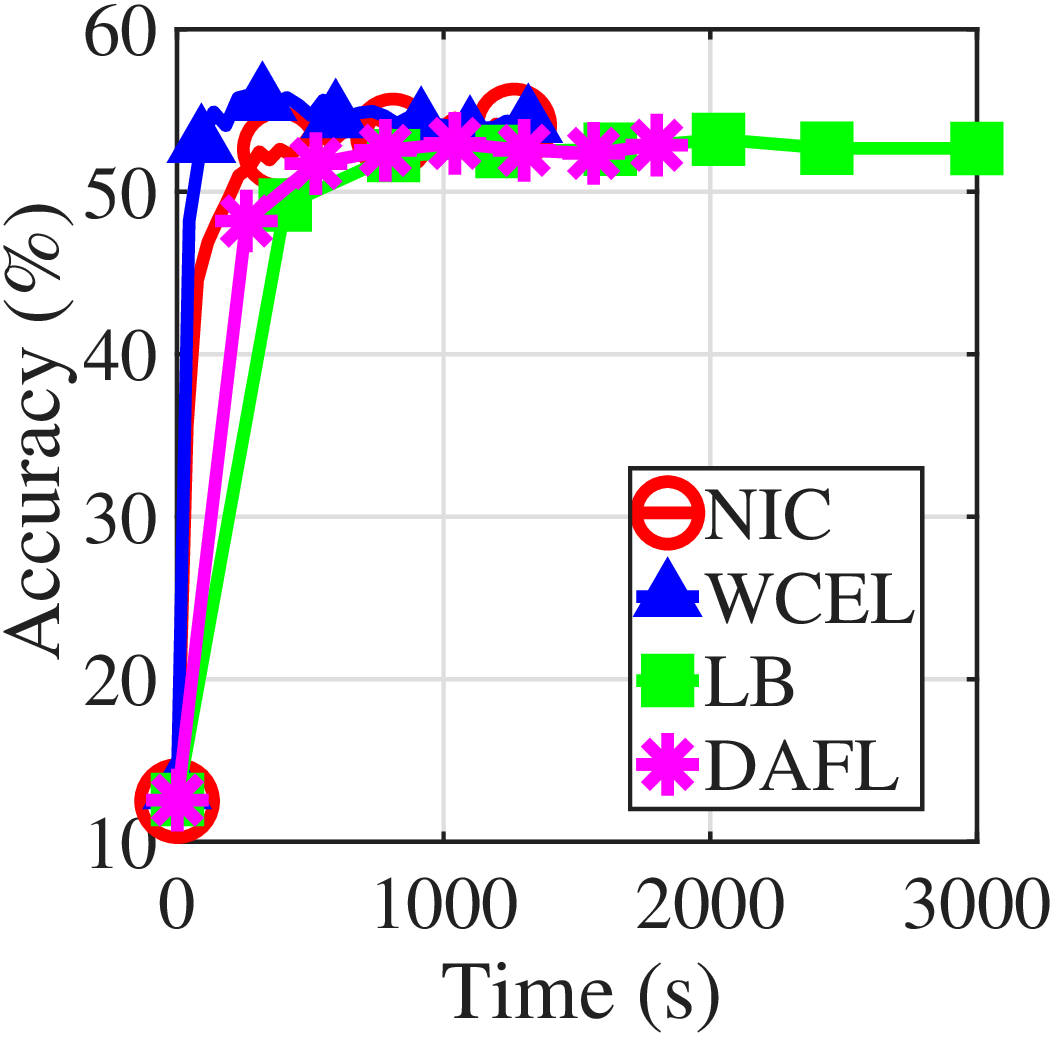}
    
    \label{fig:I_0_0_J_0_4097_accuracy}
}
\subfloat[] {
    \includegraphics[scale = 0.125]{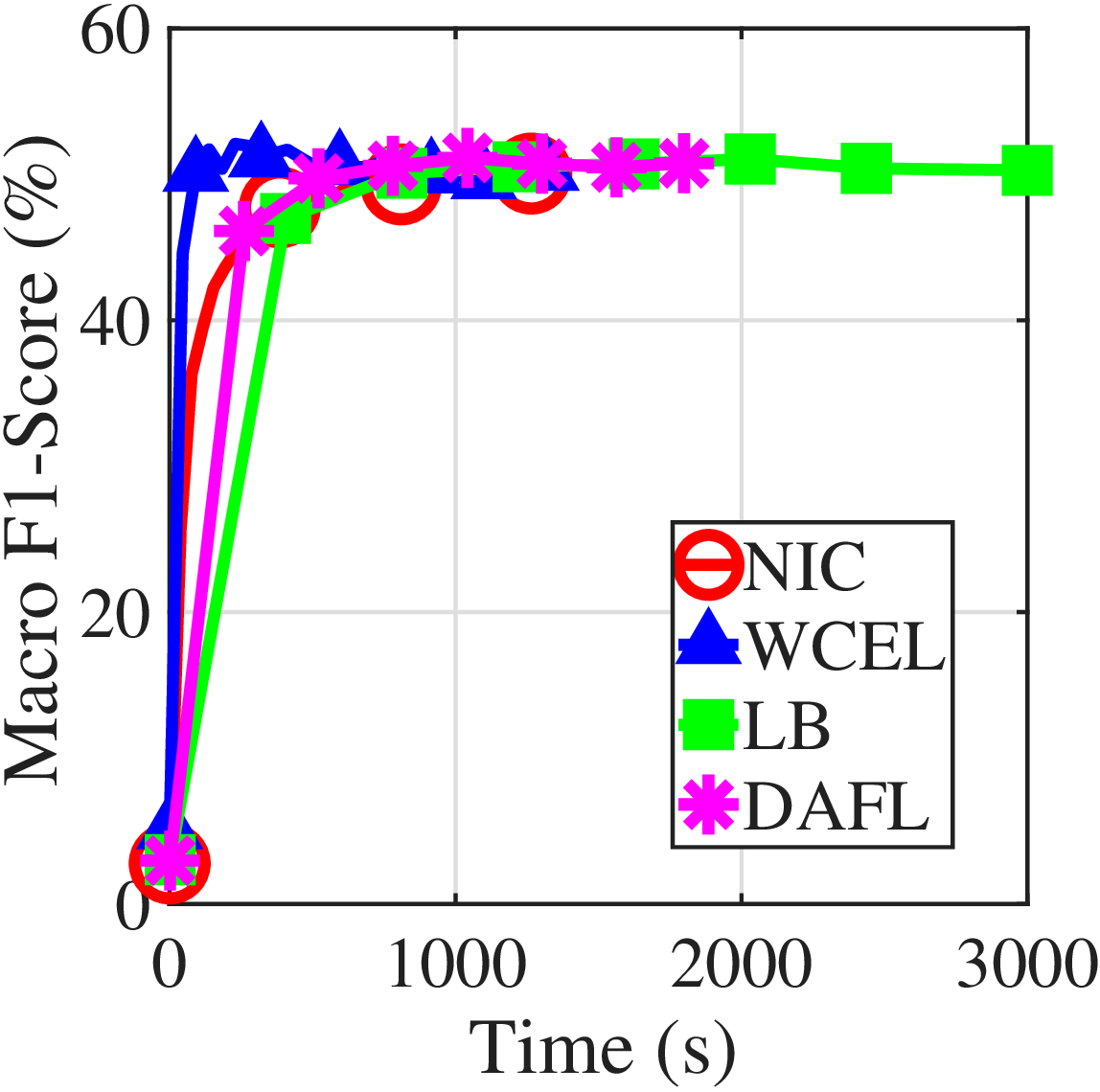}
    \label{fig:I_0_0_J_0_4097_mean_f1score}
}
\subfloat[] {
    \includegraphics[scale = 0.125]{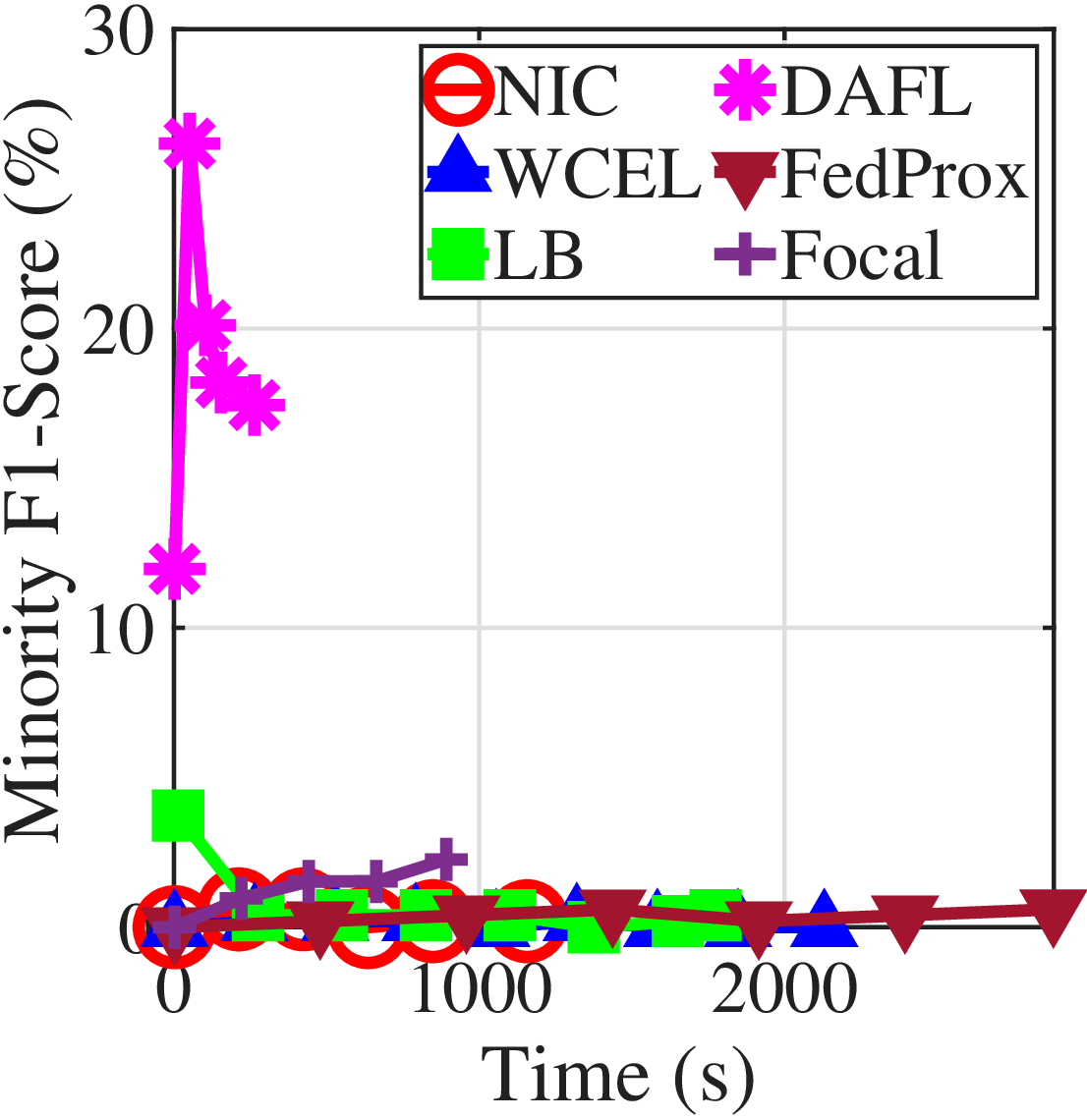}  
    \label{fig:I_0_0713_J_0_4097_minority_f1score}
}
\subfloat[] {
    
    \includegraphics[scale = 0.14]{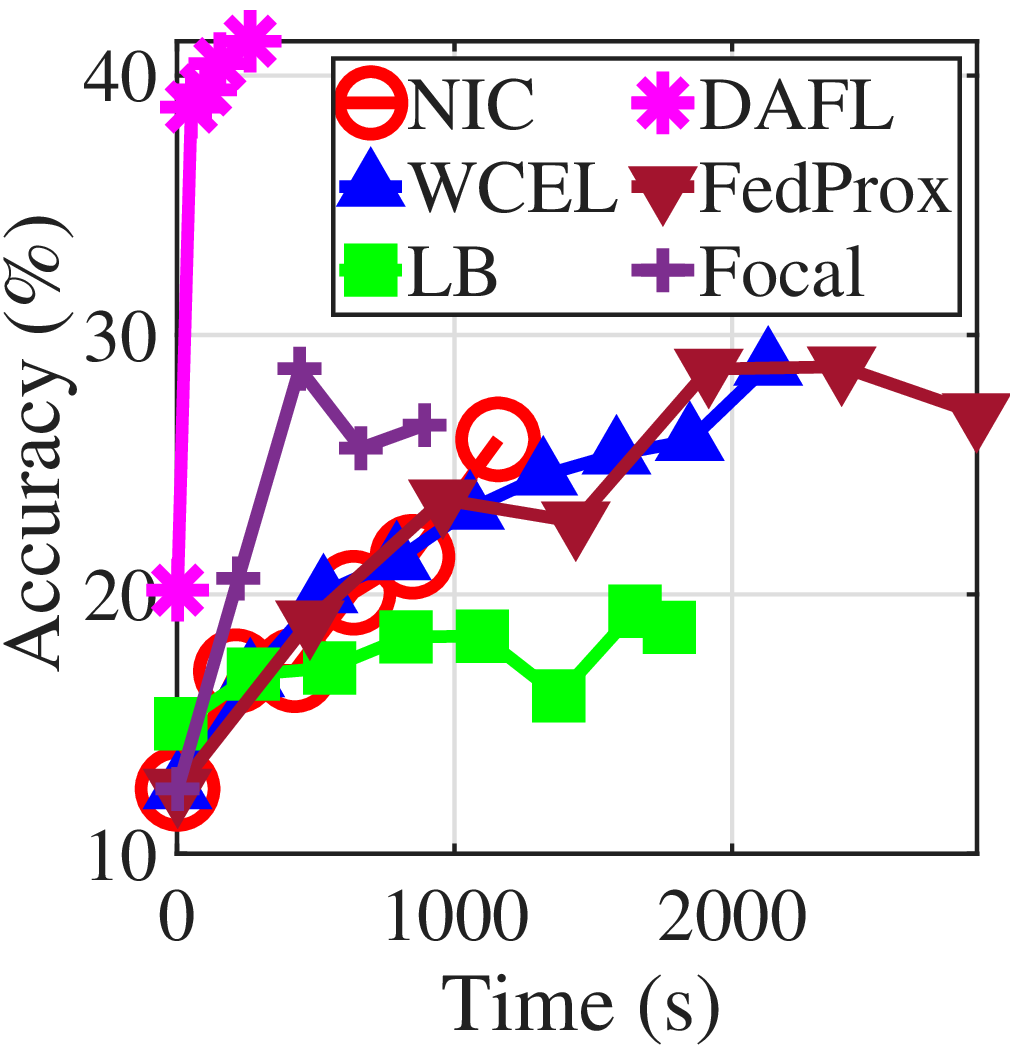}
    \label{fig:I_0_0713_J_0_4097_accuracy}
}
\subfloat[] {
    \includegraphics[scale = 0.125]{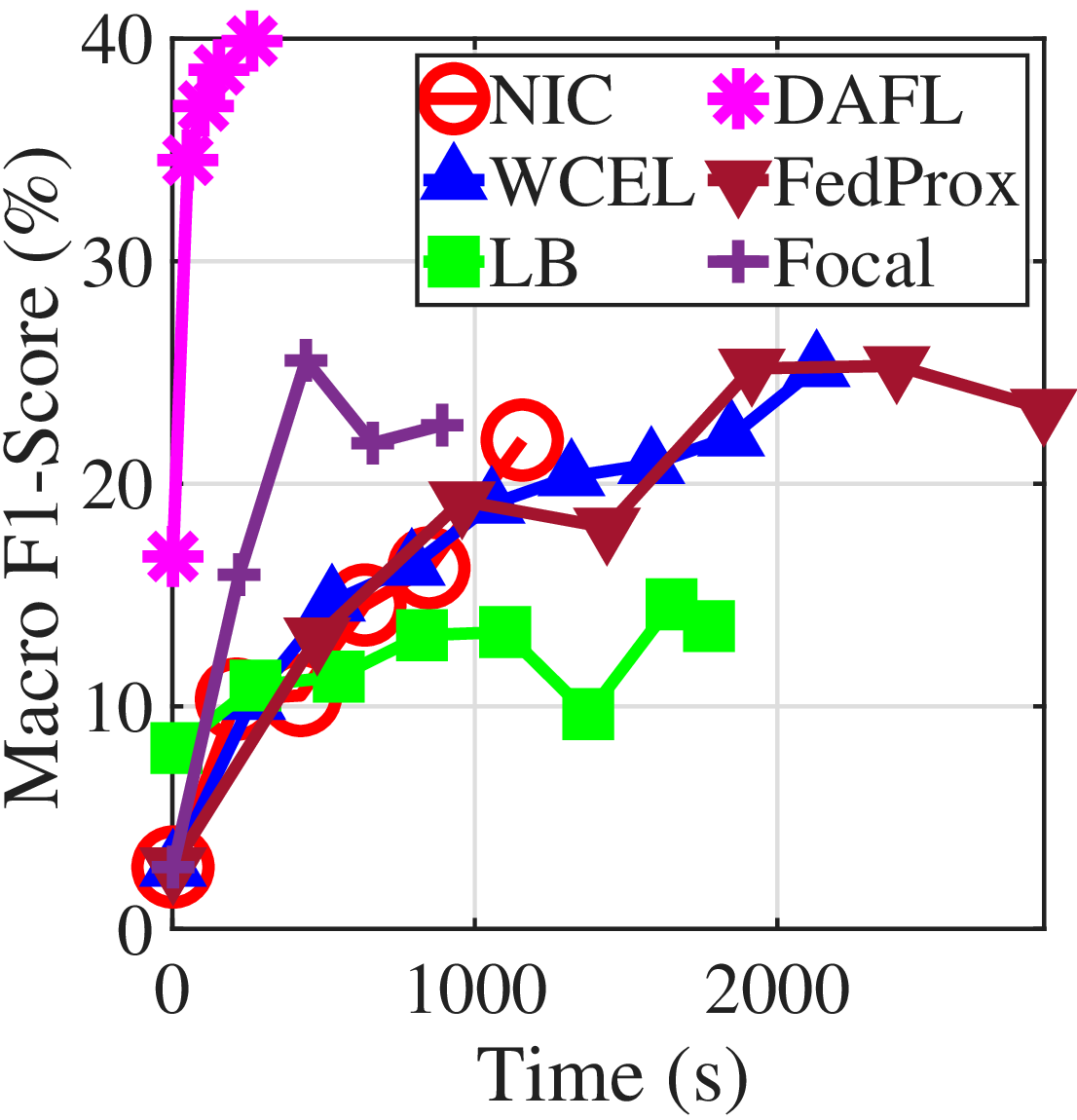}
   \label{fig:I_0_0713_J_0_4097_mean_f1score}
}

\subfloat[] {
    \includegraphics[scale = 0.14]{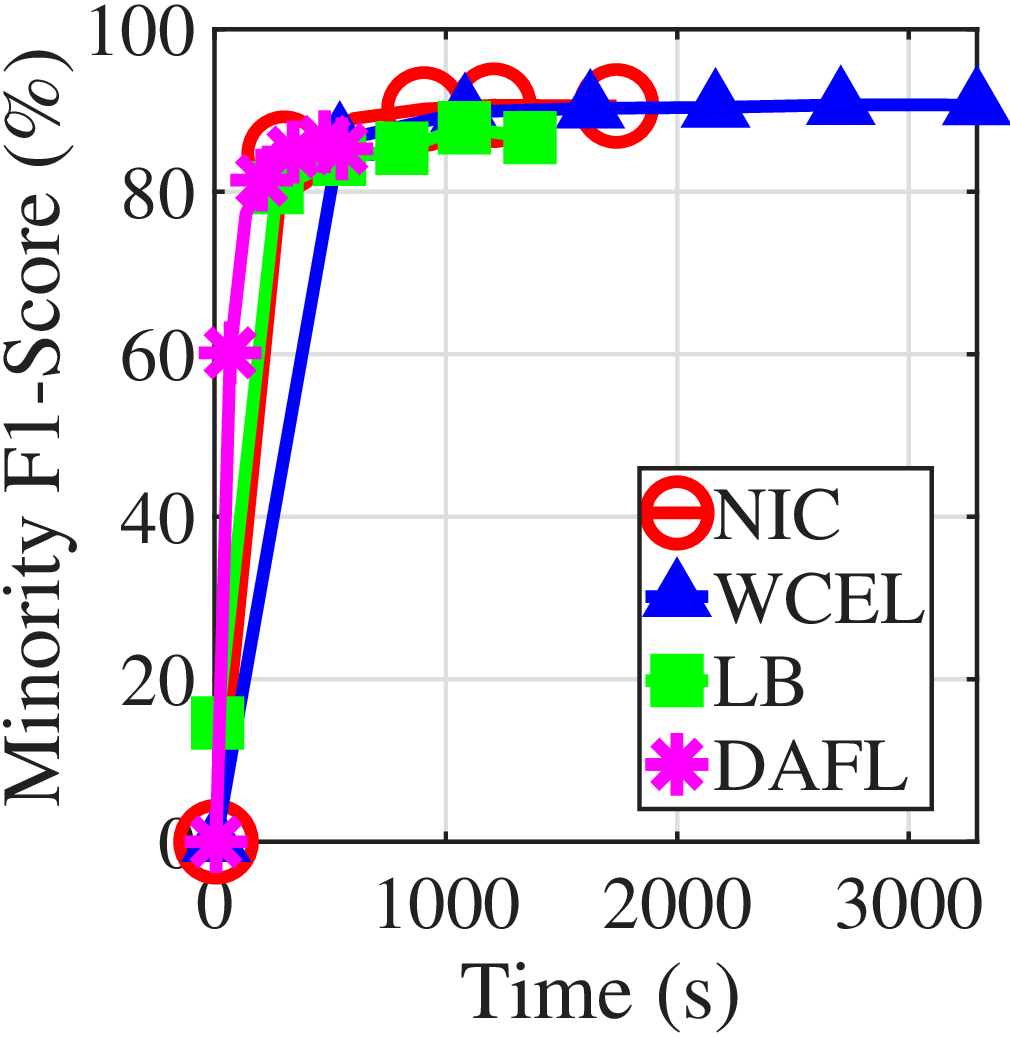}
    \label{fig:I_0_0713_J_0_001_minority_f1score}
}
\subfloat[] {
    \includegraphics[scale = 0.14]{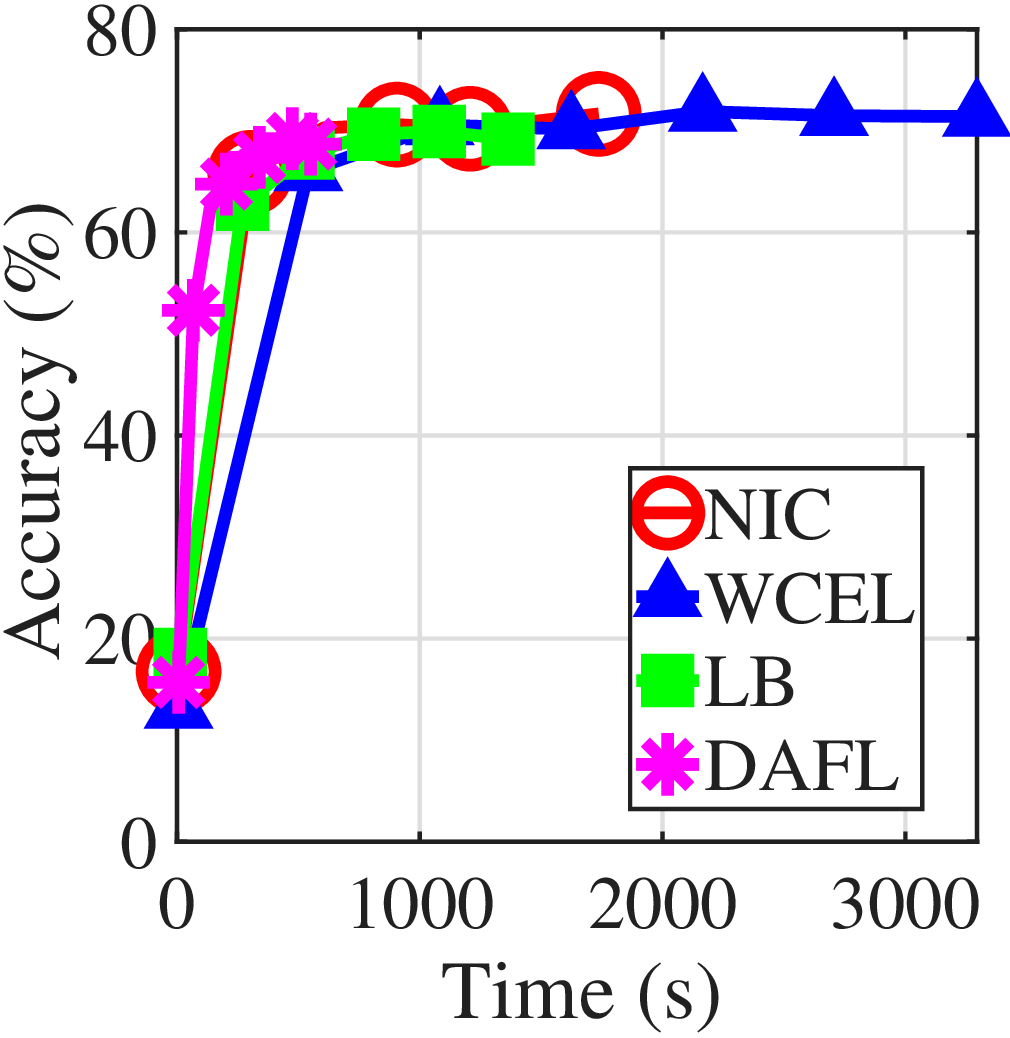}
    \label{fig:I_0_0713_J_0_001_accuracy}
}
\subfloat[] {
    \includegraphics[scale = 0.14]{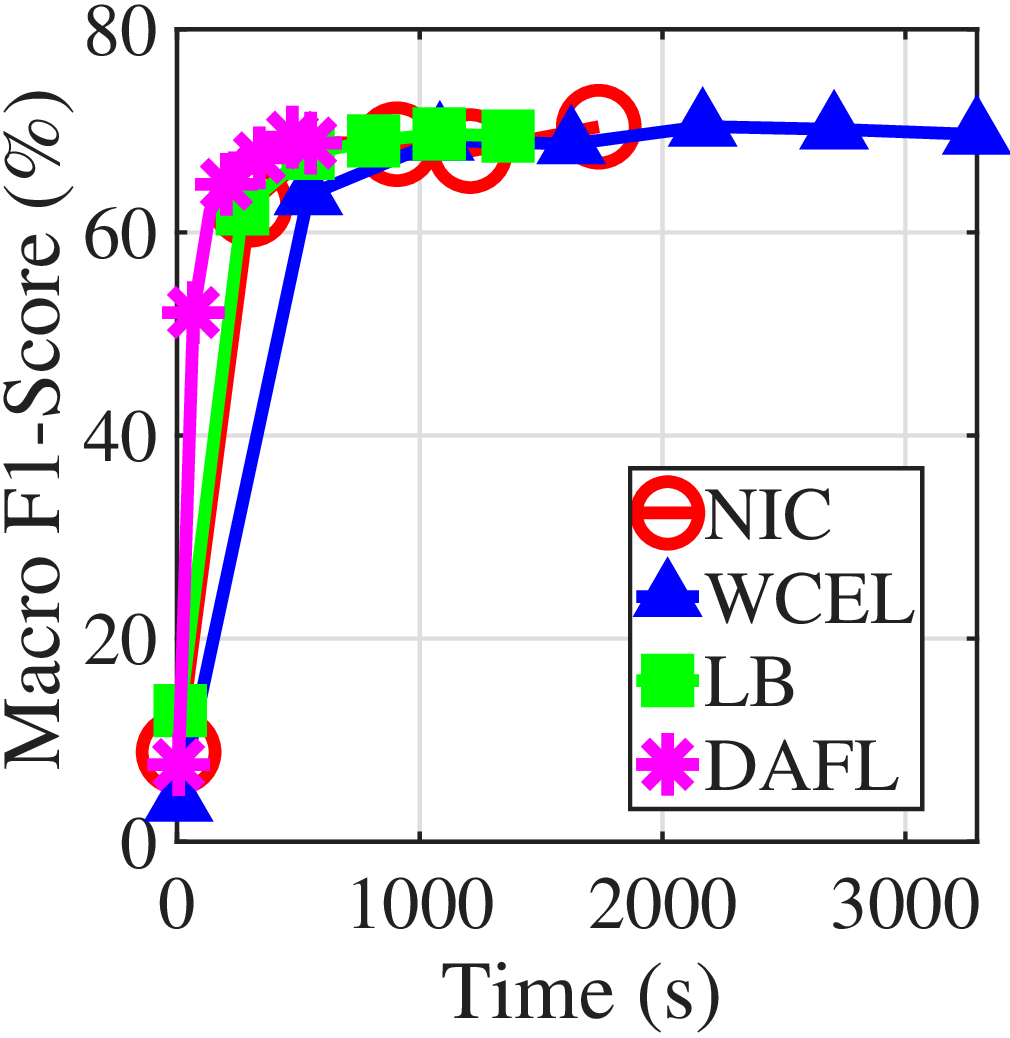} 
    \label{fig:I_0_0713_J_0_001_mean_f1score}
}
\subfloat[] {  
    \includegraphics[scale = 0.14]{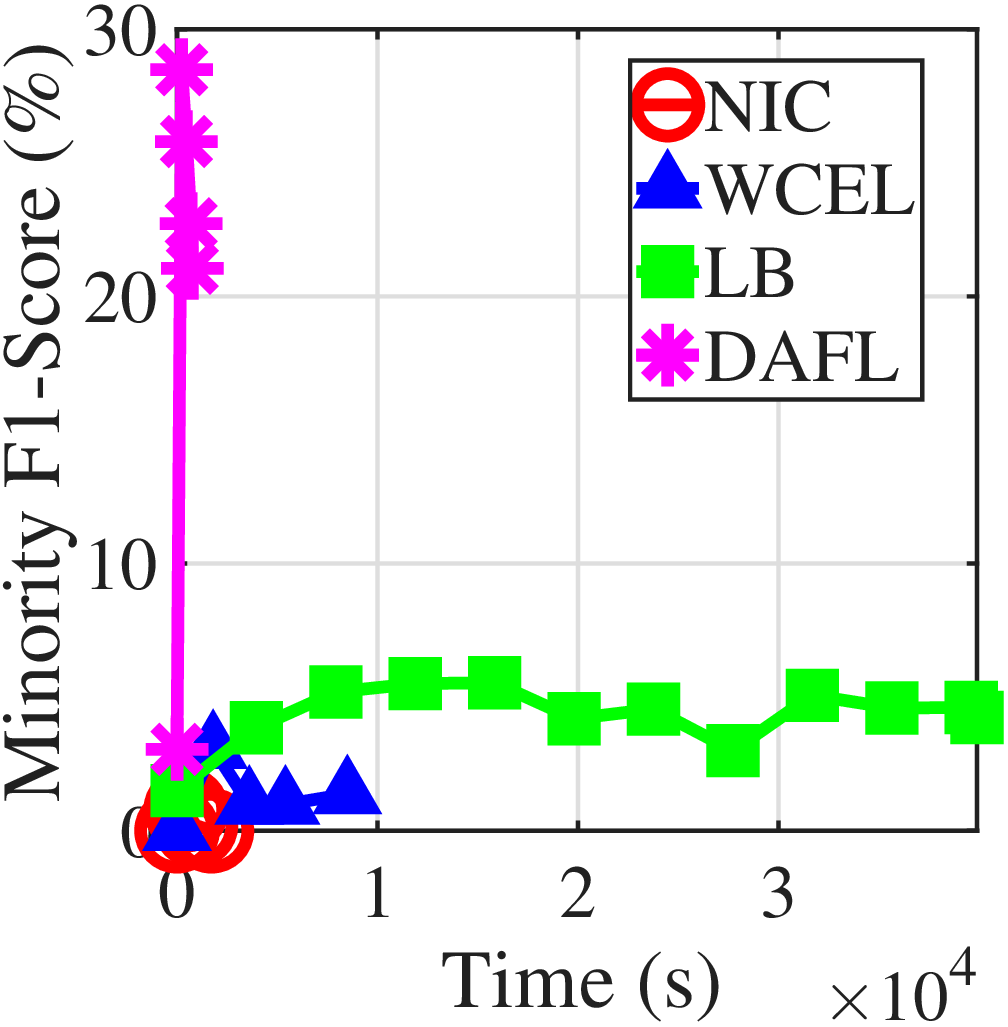}
    \label{fig:I_0_0322_J_0_4037_minority_f1score}
}
\subfloat[] {
    \includegraphics[scale = 0.14]{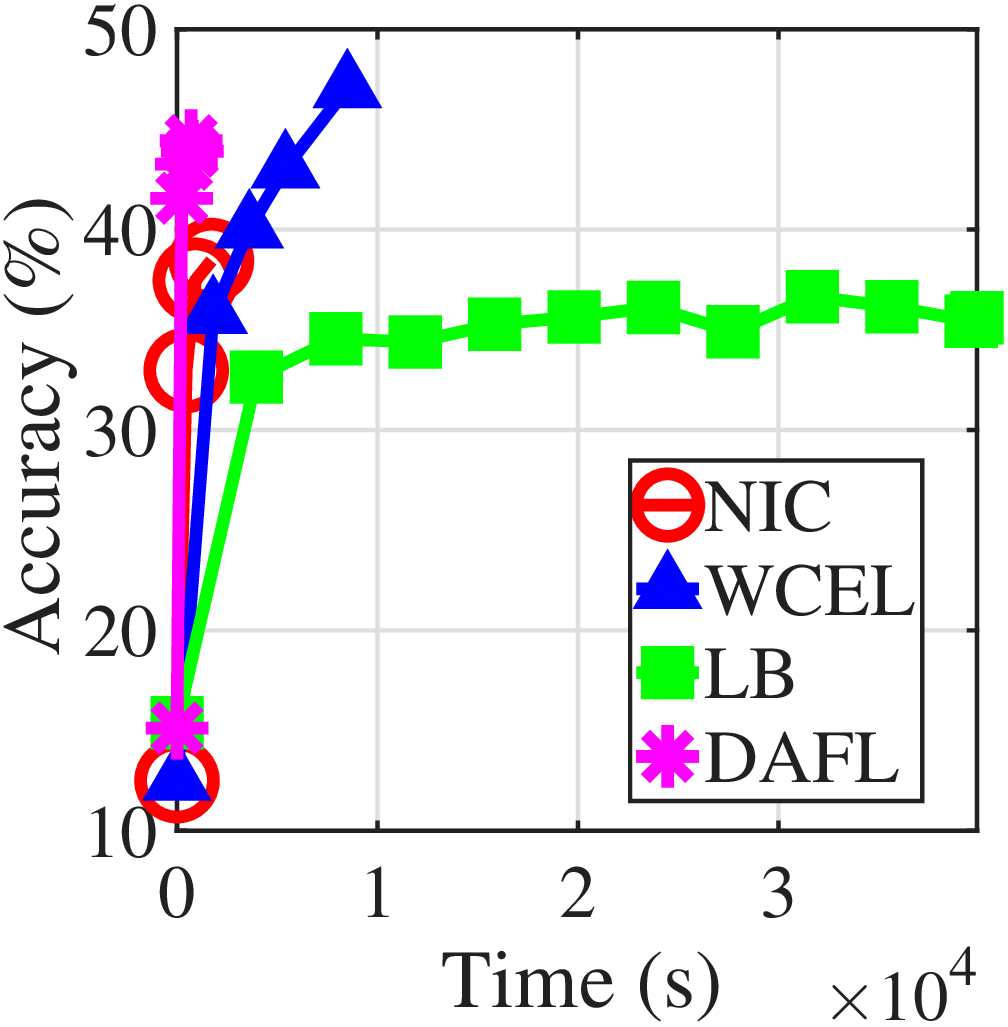}
    \label{fig:I_0_0322_J_0_4037_accuracy}
}
\subfloat[] {
    \includegraphics[scale = 0.14]{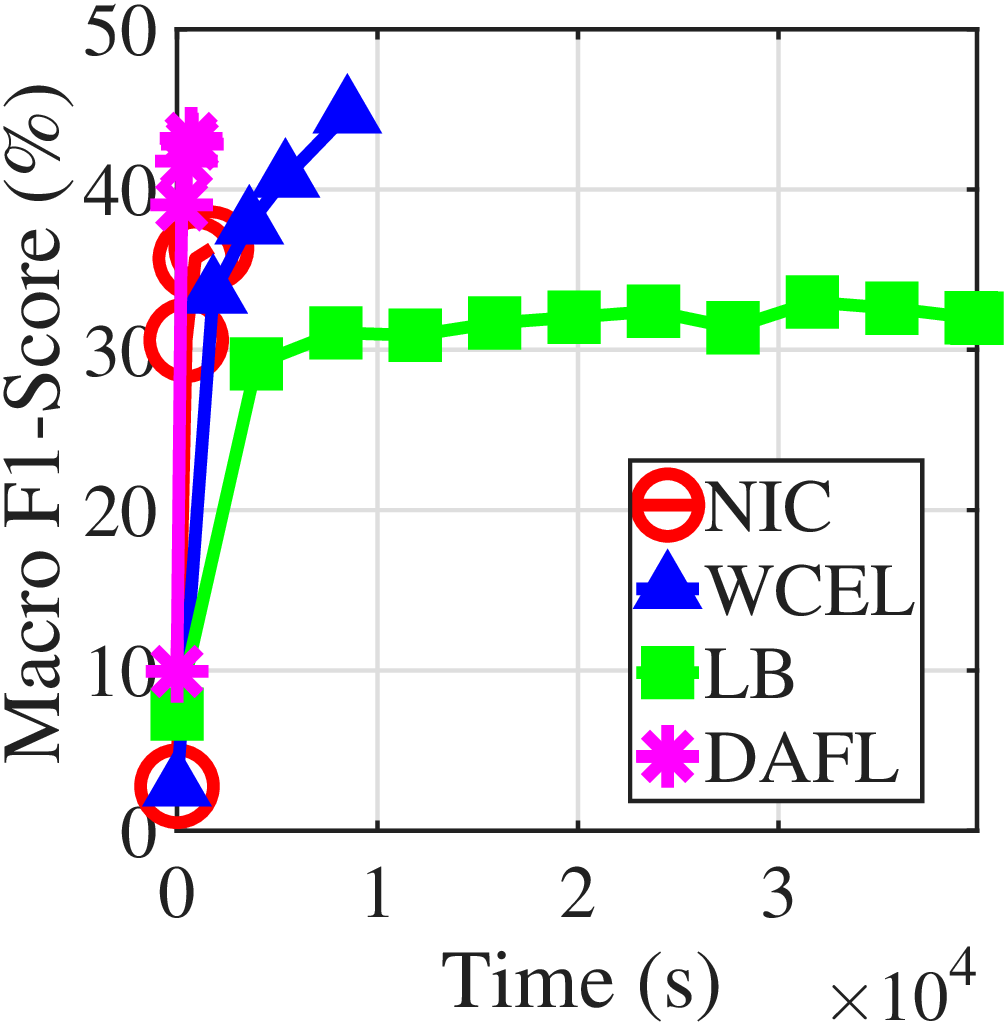}  
    \label{fig:I_0_0322_J_0_4037_mean_f1score}
}

\subfloat[] {
    \includegraphics[scale = 0.14]{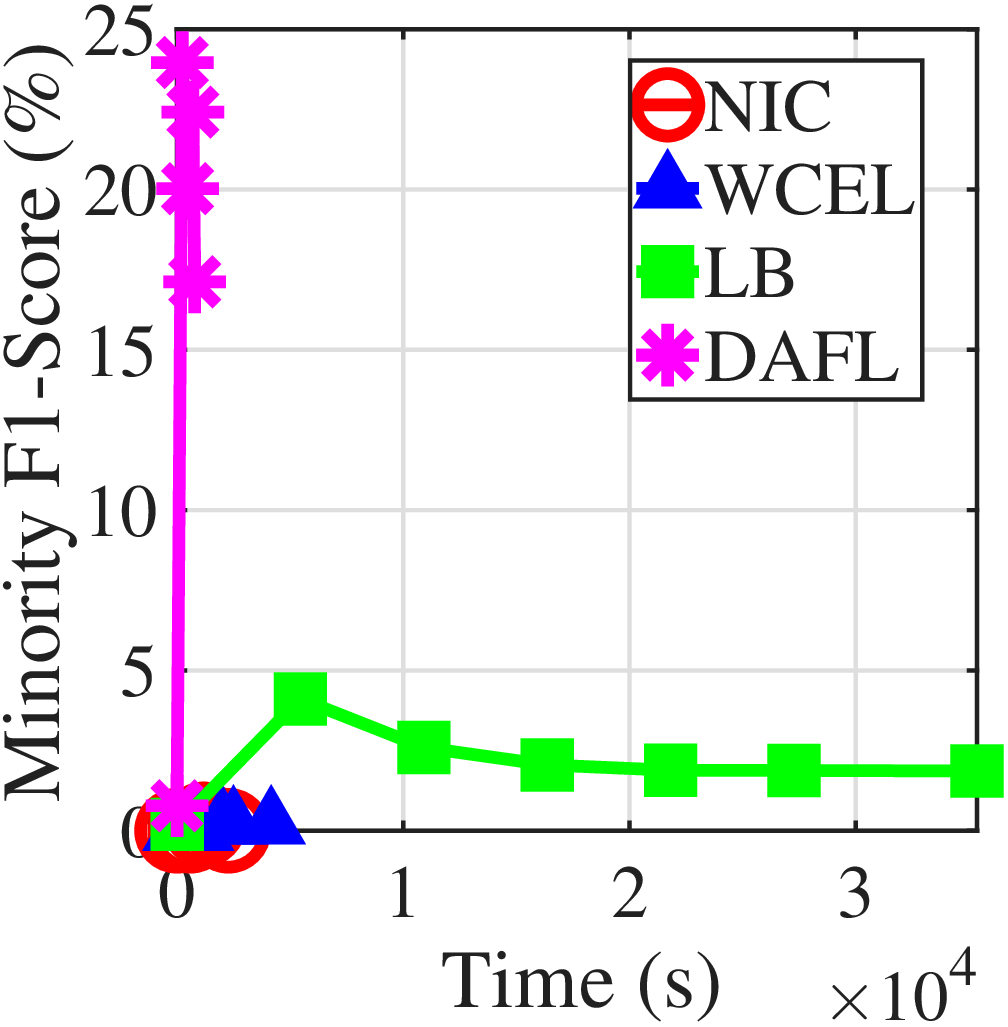}
    \label{fig:I_0_0627_J_0_3657_minority_f1score}
}
\subfloat[] {   
    \includegraphics[scale = 0.14]{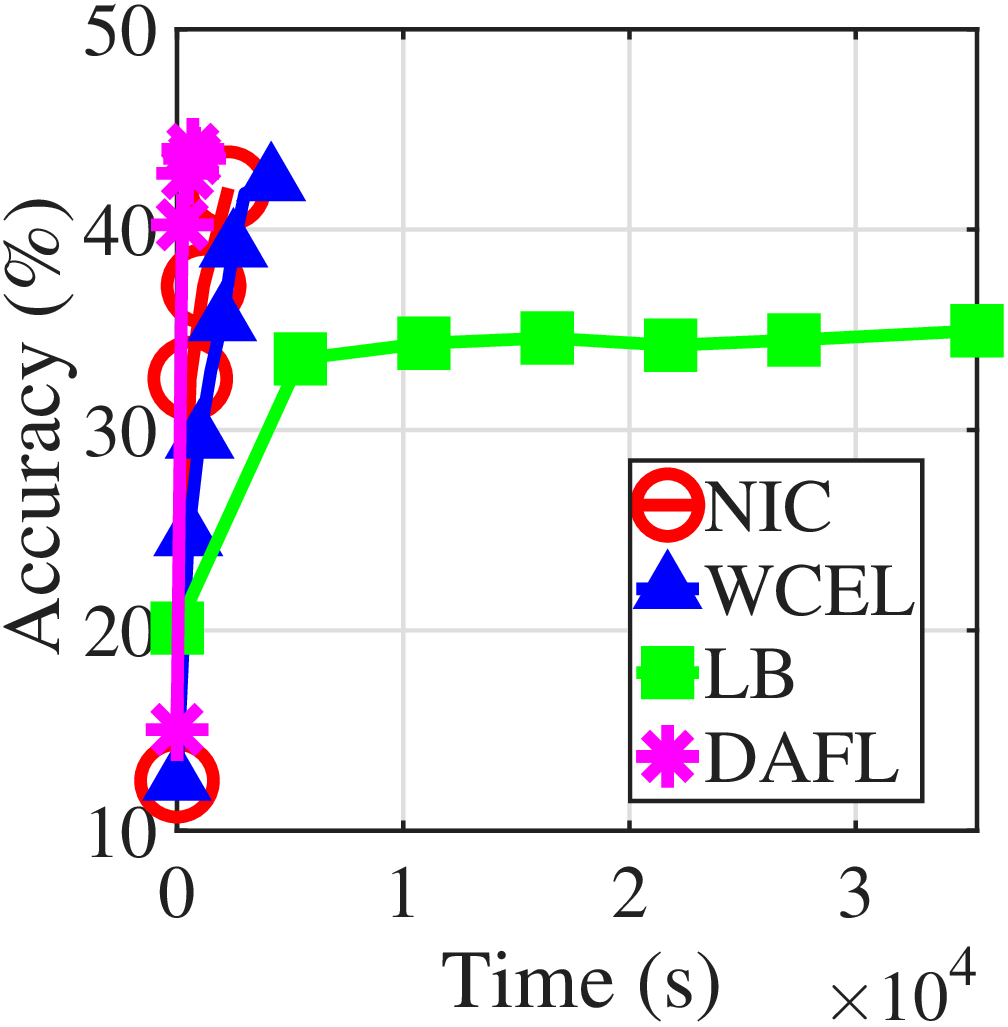}   
    \label{fig:I_0_0627_J_0_3657_accuracy}
}
\subfloat[] {
    \includegraphics[scale = 0.14]{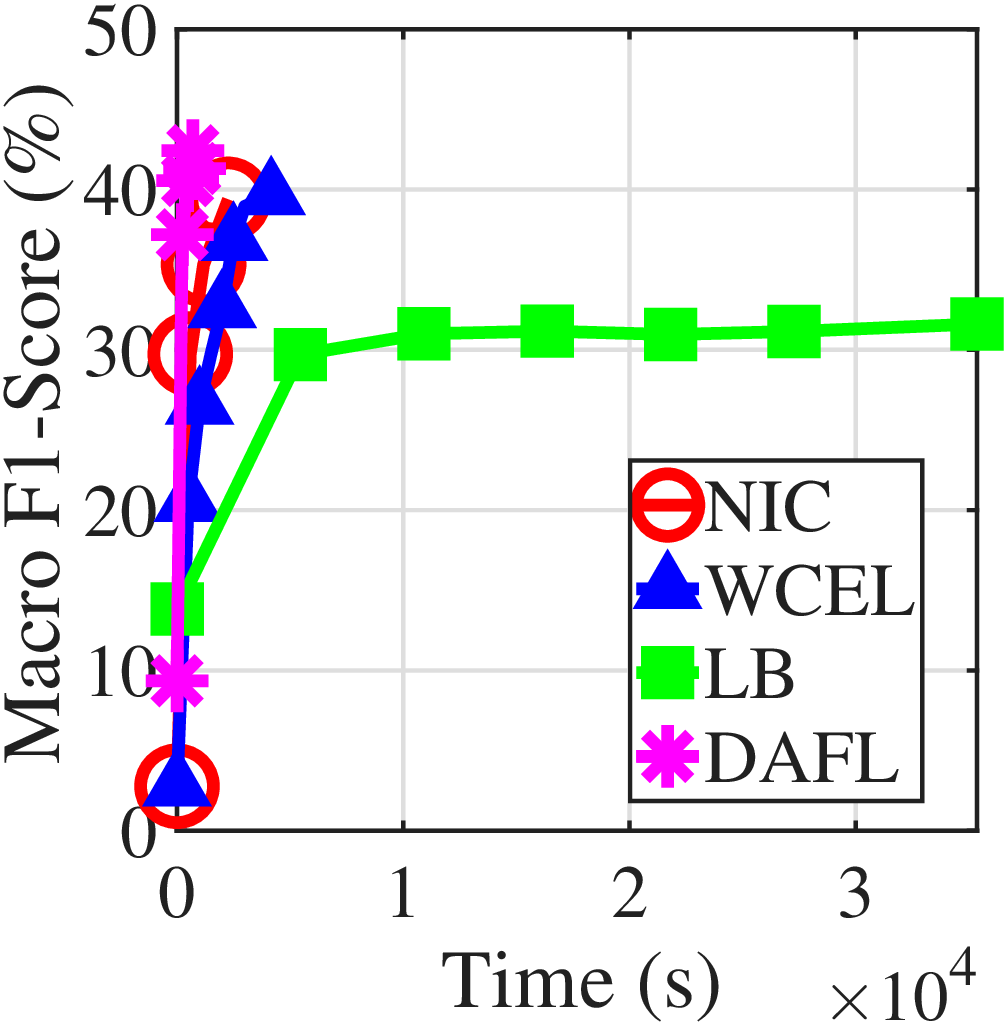}
    \label{fig:I_0_0627_J_0_3657_mean_f1score}
}
\subfloat[] {
    \includegraphics[scale = 0.14]{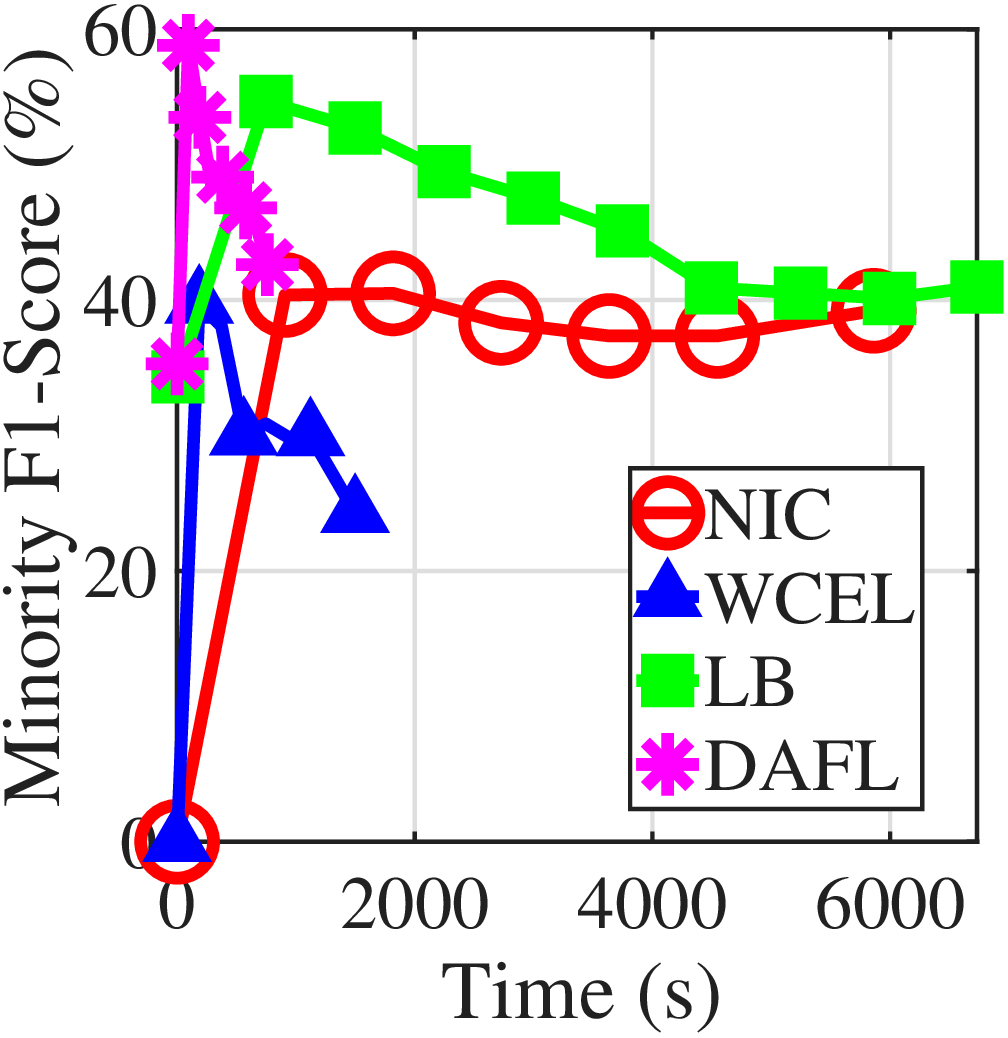}
    \label{fig:I_0_089_J_0_1309_minority_f1score}
}
\subfloat[] {
    \includegraphics[scale = 0.14]{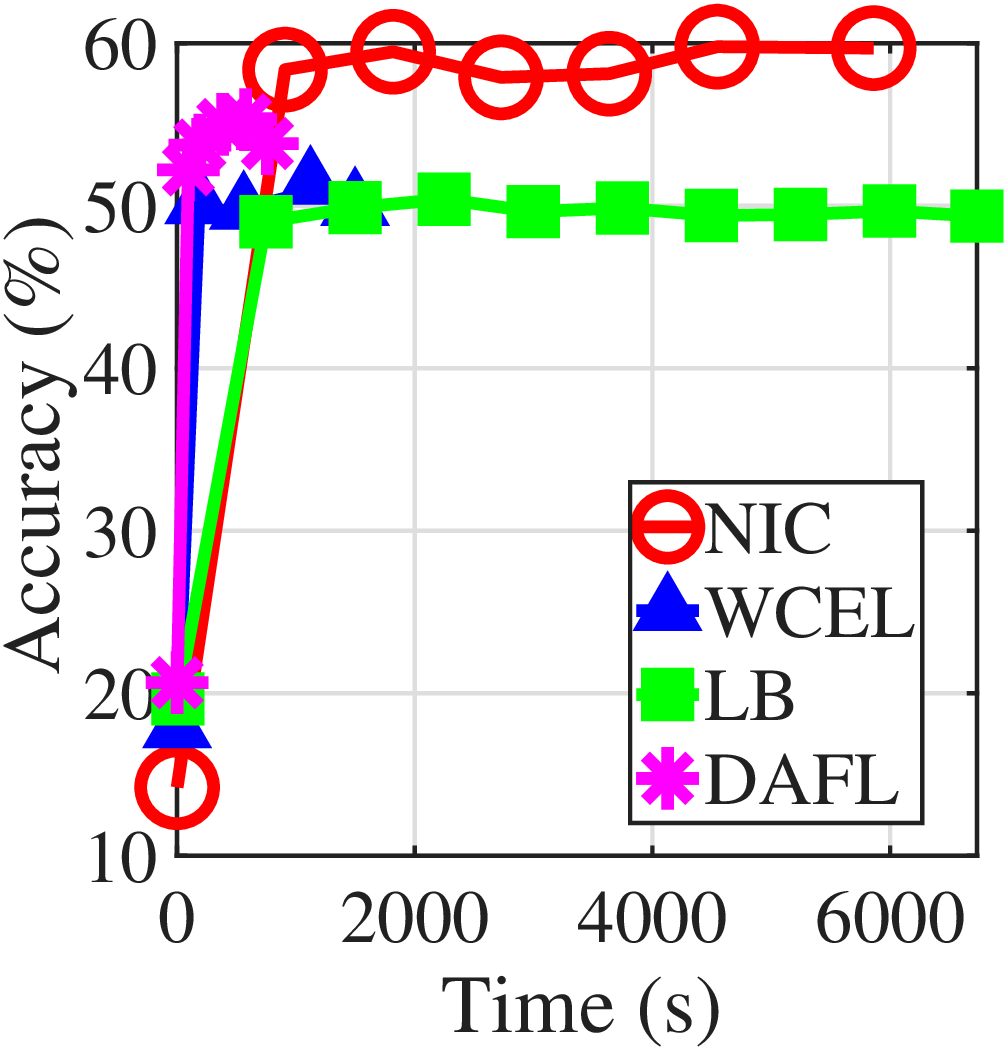}
    \label{fig:I_0_089_J_0_1309_accuracy}
}
\subfloat[] {
    \includegraphics[scale = 0.14]{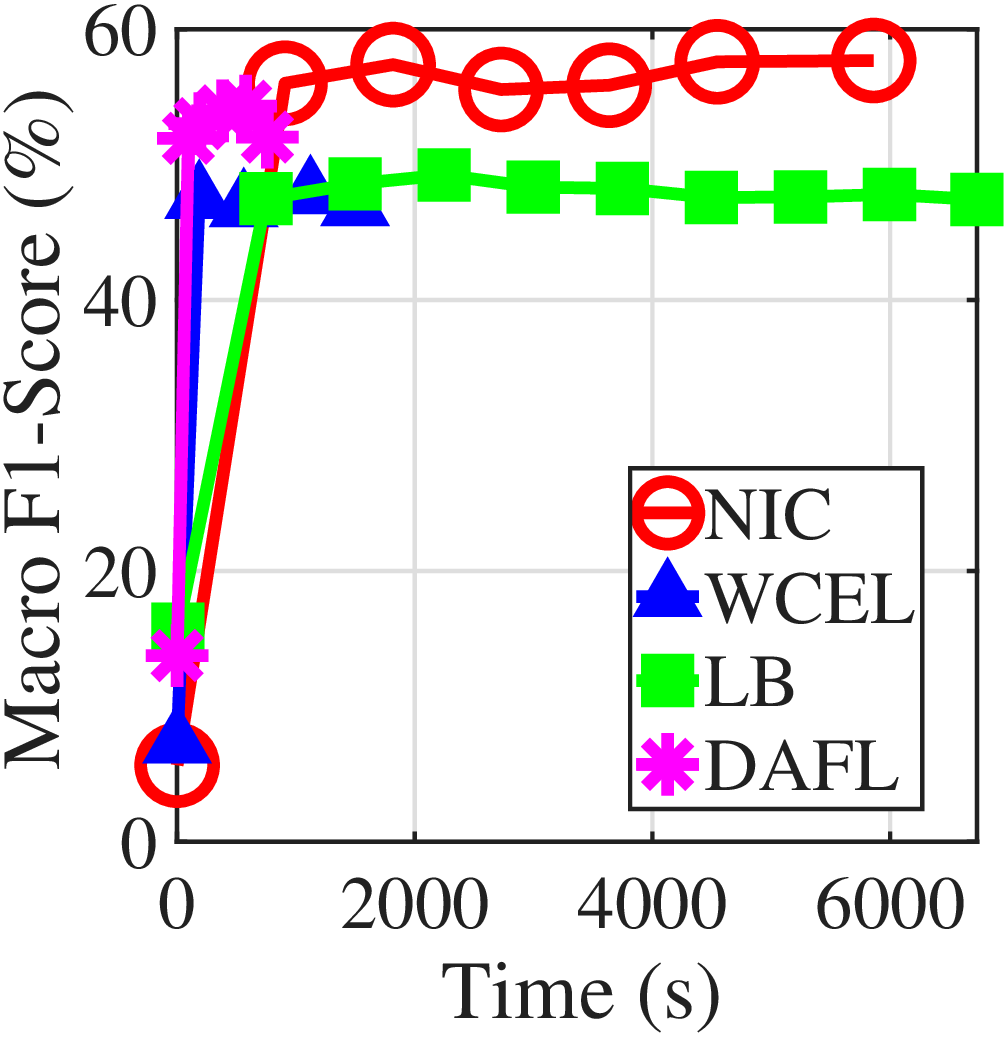}
    \label{fig:I_0_089_J_0_1309_mean_f1score}
}

\caption{Performance on datasets D1 (a)–(c), D2 (d)–(f), D3 (g)–(i), D4 (j)–(l), D5 (m)–(o), and D6 (p)–(r), with each triplet showing minority-class F1-score, accuracy, and macro F1-score, respectively.}
\vspace{-0.21cm}
\label{fig:datasetresults}
\end{figure*}
\begin{table*}[t]
\centering

\begin{minipage}[t]{0.684\linewidth}

\caption{Comparison of NIC, WCEL, LB, and DAFL across datasets D1–D6 in terms of the number of rounds and training time at the slowest client to convergence, and the resulting minority-class F1-score, accuracy, and macro F1-score on the test set.}
\label{tab:comparison}
\centering
    \tinyscript
\setlength{\tabcolsep}{1pt}

\begin{tabular}{|l|l|c|c|c|c|c|c|}
\hline
Dataset & Metric & D1 & D2 & D3 & D4 & D5 & D6 \\
\hline
\multirow{4}{*}{Rounds} & NIC & 2470.333±429.784 & 420.000±149.880 & 182.667±75.514 & 1218.333±310.989 & 1663.333±289.856 & 2517.667±1843.102 \\
 & WCEL & 3978.667±758.043 & 1466.333±390.070 & 552.333±115.106 & 6570.000±1309.133 & 2159.000±209.695 & 1402.333±426.602 \\
 & LB & 1495.667±1189.110 & 155.333±24.705 & 80.667±22.189 & 4527.333±451.948 & 2909.000±402.031 & 988.000±471.487 \\
 & DAFL & 1441.667±971.877 & 91.333±40.415 & 80.333±9.018 & 484.333±123.144 & 952.667±618.361 & 493.000±267.185 \\
\hline
\multirow{4}{*}{Time (s)} & NIC & 1058.379±163.601 & 899.667±307.991 & 1225.086±533.639 & 1555.357±398.398 & 2715.395±480.683 & 3179.173±2263.749 \\
 & WCEL & 1686.528±355.184 & 3070.904±781.052 & 3690.965±755.078 & 8436.223±1664.224 & 3743.674±279.595 & 1715.314±463.484 \\
 & LB & 3037.262±2428.344 & 2136.086±377.024 & 1089.099±299.965 & 35869.234±4244.298 & 31736.594±4284.032 & 7403.956±3519.098 \\
 & DAFL & 2453.844±1566.553 & \textbf{188.119±86.344} & \textbf{545.858±71.753} & \textbf{794.584±206.494} & \textbf{1550.310±999.016} & \textbf{878.630±474.515} \\
\hline
\multirow{4}{*}{\makecell{Minority\\F1}} & NIC & 0.008±0.002 & 0.001±0.001 & 0.896±0.011 & 0.004±0.005 & 0.001±0.002 & 0.390±0.037 \\
 & WCEL & 0.007±0.002 & 0.000±0.000 & 0.904±0.010 & 0.007±0.006 & 0.001±0.001 & 0.278±0.030 \\
 & LB & 0.152±0.041 & 0.005±0.003 & 0.859±0.007 & 0.036±0.010 & 0.018±0.001 & 0.364±0.046 \\
 & DAFL & \textbf{0.160±0.048} & \textbf{0.183±0.055} & 0.856±0.014 & \textbf{0.253±0.037} & \textbf{0.147±0.028} & \textbf{0.436±0.017} \\
\hline
\multirow{4}{*}{Accuracy} & NIC & 0.535±0.006 & 0.218±0.036 & 0.690±0.024 & 0.380±0.005 & 0.397±0.026 & 0.579±0.016 \\
 & WCEL & 0.540±0.003 & 0.303±0.013 & 0.721±0.021 & 0.442±0.025 & 0.403±0.019 & 0.499±0.001 \\
 & LB & 0.521±0.015 & 0.175±0.011 & 0.681±0.010 & 0.339±0.016 & 0.343±0.006 & 0.483±0.010 \\
 & DAFL & 0.516±0.012 & 0.406±0.007 & 0.689±0.002 & 0.451±0.010 & 0.437±0.026 & 0.544±0.014 \\
\hline
\multirow{4}{*}{Macro F1} & NIC & 0.495±0.006 & 0.170±0.043 & 0.669±0.031 & 0.360±0.003 & 0.368±0.026 & 0.555±0.019 \\
 & WCEL & 0.501±0.003 & 0.264±0.011 & 0.707±0.026 & 0.419±0.025 & 0.374±0.019 & 0.464±0.005 \\
 & LB & 0.495±0.015 & 0.117±0.017 & 0.684±0.010 & 0.299±0.023 & 0.312±0.006 & 0.462±0.012 \\
 & DAFL & 0.491±0.017 & 0.392±0.008 & 0.689±0.001 & 0.445±0.015 & 0.413±0.029 & 0.530±0.015 \\
\hline
\end{tabular}

\end{minipage}
\hfill
\begin{minipage}[t]{0.3\linewidth}

    \caption{Comparison of FedProx, Focal Loss, and DAFL on D2. The number of rounds and training time for convergence, and the minority-class F1-score, accuracy, and macro F1-score.}
    \label{tab:D2comparison}
    \centering
    \tinyscript
    \setlength{\tabcolsep}{1pt}
    \begin{tabular}{|c|c|c|}
        \hline
        Dataset & Metric & D2 \\
        \hline
        
        \multirow{4}{*}{Rounds}
        & FedProx           & 1740.667±1373.816 \\
        & Focal            & 262.000±124.756 \\
        & DAFL           & 91.333±40.415 \\
        \hline
        \multirow{4}{*}{Time (s)}
        & FedProx      & 4167.974±3232.481 \\
        & Focal        & 578.554±271.105 \\
        & DAFL         & \textbf{188.119±86.344}\\
        \hline
        \multirow{4}{*}{Minority F1}
        & FedProx  & 0.002±0.003 \\
        & Focal  & 0.023±0.019 \\
        & DAFL  & \textbf{0.183±0.055} \\
        \hline
        
        \multirow{4}{*}{Accuracy}
        & FedProx  & 0.293±0.054 \\
        & Focal    & 0.229±0.033 \\
        & DAFL     & 0.406±0.007 \\
        \hline
        
        \multirow{4}{*}{Macro F1}
        & FedProx  & 0.255±0.054 \\
        & Focal    & 0.182±0.038 \\
        & DAFL     & 0.392±0.008 \\
        \hline
    \end{tabular}
\end{minipage}

\end{table*}

\section{Experimental Results}\label{sec:expres}
     We use a modified version of the ConvNet architecture proposed in \cite{ConvNet}. It has six convolutional layers followed by three fully connected layers, with 1,343,146 learnable parameters.
     Local training is performed using SGD with a learning rate ranging from 0.003 to 0.3 for 2 epochs with a batch size of 64, and aggregation is performed using FedAvg. DAFL is implemented in PyTorch and trained and evaluated on a system equipped with an NVIDIA T1000 8 GB GPU.
    
     We conduct our experiments with $\zeta\! =\! 1$ and $\zeta'\! =\! 2$. The augmentation phase is conducted with $\epsilon\! =\! 0.001$, $\tau\! =\! 0.001$, $\lambda_1\! =\! 1$, $s\! =\! 1$ and $\lambda_2\! \in\! \{10^2, 10^3, 10^4\}$. For all training experiments, we use $\lambda_2\! = \!10^2$. We set $\epsilon\! =\! 0.001$ so that the $\mathcal{I}$-dependent term in the convergence bound remains small, and $\tau\! =\! 0.001$ to avoid the cost of exact balancing while avoiding substantial residual global imbalance. However, we do not tune $\epsilon$ and $\tau$.
     Since $L^2CG^2$ is an unknown scaling constant, it is set to 1 in all experiments. We define convergence based on the stability of the training loss.
     Specifically, training is considered to have converged when the difference in training loss over five consecutive rounds is less than 0.001, and the corresponding training duration is reported as the convergence time. We do not use a common performance target since several baselines do not achieve a useful minority-class F1-score under severe imbalance.
     Unless stated otherwise, all experiments are repeated over three random seeds, and the results are reported as mean ± standard deviation. Minority-class performance is evaluated using the minority-class F1-score which refer to F1-score of the minority-class.
     We compare DAFL against the following baselines: (i) 
     \textit{No Imbalance Correction (NIC)}, (ii) \textit{Weighted Cross Entropy Loss (WCEL)}, (iii) \textit{Local Balancing (LB)},  (iv) \textit{FedProx}, and (v) \textit{Focal Loss}.
    \subsection{Experimental Evaluation - Augmentation Phase}
    To augment additional samples at each client, we use low-complexity methods such as rotation within $[-30^{\circ},30^{\circ}]$ and horizontal flipping. For the optimization problem, we fix $\lambda_1$ and $\epsilon$, and vary $\lambda_2 \in \{10^2,10^3,10^4\}$. 
    Table \ref{tab:updated_dataset_details} reports the updated $\mathcal{I}$ and $J$ for LB and DAFL across $\lambda_2\! \in\! \{10^2, 10^3, 10^4\}$.
    The updated dataset obtained using DAFL satisfies $J \leq \tau$.  Additionally,  DAFL aligns the local distribution with the global distribution to the extent required to minimize the optimization objective, hence, $\mathcal{I}$ is not necessarily close to zero. At the same time, since LB needs every client to satisfy $J_k \le \tau$, the clients' local distributions are forced to match the global distribution, yielding $\mathcal{I} \approx 0$, and consequently,  $J \leq \tau$. In addition to the resulting imbalance measures, Table \ref{tab:updated_dataset_details} also reports the total number of samples augmented in the federation, the time taken for augmentation, which is computed as the time taken by the slowest client to perform augmentation, and the additional time taken at the server to compute $\hat{\mathbf{A}}$. From this table, we have the following observations. First, the augmentation time for DAFL is lower than that of LB. However, we point out that this is offset by the added server-side computation time, which must be compensated for by the reduction in training time. 
    Second, when $J\le\tau$ and $\mathcal I\gg0$, DAFL may require more augmented samples than LB because the greedy optimization algorithm does not guarantee a globally optimal solution. 
    Finally, to visualize this, Fig. \ref{fig:updated_dataset2} shows the heatmap of client-class counts for D2 before augmentation, after LB, and after DAFL. In particular, Fig. \ref{fig:Dataset2_N} shows that D2 is highly imbalanced, and  Fig. \ref{fig:Dataset2_LBJ_Nprime} shows that every client is locally balanced, resulting in a global balance. Fig. \ref{fig:Dataset2_Nprime} shows that, under DAFL, the dataset becomes globally balanced and the local distributions align more closely with the global distribution than in the original dataset even though local balance is not achieved. This observation is in alignment with the results tabulated in  Table \ref{tab:updated_dataset_details}, wherein for D2 with $\lambda_2 = 100$, DAFL reduces $\mathcal{I}$ from $0.3982$ to $0.1479$ and $J$ from $0.4097$ to $0.001$. 
    
    \subsection{Experimental Evaluation - FL Phase}
    \begin{figure*}
            \begin{footnotesize}
            \begin{IEEEeqnarray} {lcl} \label{eqn:conditional_expectation_on_smoothness_inequality}
            \mathbb{E}_t [F(\mathbf{w}^{(t+1)})]
            \le
            F(\mathbf{w}^{(t)})
            -
            \eta \langle 
            \nabla F (\mathbf{w}^{(t)}),
            \mathbb{E}_t[\mathbf{g}^{(t)}]
            \rangle   +
            \frac{L\eta^2}{2} \mathbb{E}_t
             [ \|\mathbf{g}^{(t)}\|^2 ]. 
             \\
             \nabla F_k(\mathbf w) - \nabla F(\mathbf w) = \underbrace{\textstyle\sum_c(\pi_k(c)-\pi(c))\nabla\bar F_c(\mathbf w)}_{=: \, u_k(\mathbf w) \ \text{(label proportion imbalance)}} + \underbrace{\textstyle\sum_c \pi_k(c)(\nabla F_{k,c}(\mathbf w) - \nabla \bar F_c(\mathbf w))}_{=: \, v_k(\mathbf w) \ \text{(conditional shift)}}. 
             \label{eqn:hetero-decomp}
             \\
             \mathbb E_t\|\mathbf{z}^{(t)}\|^2 \;\le\; (E-1)^2[8\eta^2E^2L^2\|{\nabla F(\mathbf w^{(t)})}\|^2+16\eta^2E^2L^2(CG^2\mathcal I+\Gamma)+2\eta^2E^2L^2\sigma^2]. 
             \label{eq:cor1}
             \\
             \frac{L\eta^2}2\mathbb E_t\|{\mathbf g^{(t)}}\|^2 \le \frac{3L\eta^2E^2}2\|{\nabla F(\mathbf w^{(t)})}\|^2 + L^3\eta^4E^2(E-1)^2(12\|{\nabla F(\mathbf w^{(t)})}\|^2 + 24(CG^2\mathcal I+\Gamma) + 3\sigma^2) + \frac{3L\eta^2E}2\sigma_g^2. 
             \label{eq:step6-d}
             \\
             \mathbb E_t[F(\mathbf w^{(t+1)})] \le F(\mathbf w^{(t)}) \;-\;[\eta E - \frac\eta8 - 16A - \frac{3L\eta^2E^2}2 - 12B]\|{\nabla F(\mathbf w^{(t)})}\|^2 + (32A+24B)(CG^2\mathcal I+\Gamma) + (4A+3B)\sigma^2 + \frac{3L\eta^2E}2\sigma_g^2. 
             \label{eq:step6-e}
             \\
             \mathbb E_t[F(\mathbf w^{(t+1)})] \;\le\;\;  F(\mathbf w^{(t)}) - \frac{\eta E}2\|{\nabla F(\mathbf w^{(t)})}\|^2 + \frac{3L\eta^2E}2\sigma_g^2 + 40\,\eta^3E^2(E-1)^2L^2(CG^2\mathcal I+\Gamma+\sigma^2).
             \label{eq:descent-final}
             \\
             \frac1T{\textstyle\sum_{t=0}^{T-1}\mathbb E\|{\nabla F(\mathbf w^{(t)})}\|^2} \;\le\; \frac1{\eta E T}2(F(\mathbf w^0)-F(\mathbf w^*)) + 3L\eta\,\sigma_g^2 + 80\,\eta^2E(E-1)^2L^2(CG^2\mathcal I+\Gamma+\sigma^2). 
             \label{eq:telescope-final}
             \\ 
             \hline\nonumber
        \end{IEEEeqnarray}
        \end{footnotesize}
        \end{figure*}
    Table \ref{tab:comparison} reports the number of rounds and training time required for convergence, along with minority-class F1-score, accuracy, and macro F1-score on the test dataset for NIC, WCEL, LB and DAFL. 
    Fig.~\ref{fig:datasetresults} further plots the minority-class F1-score, accuracy, and macro F1-score against training time for the run with random seed 42. 
    From Table~\ref{tab:comparison} and Fig.~\ref{fig:datasetresults}, three key observations can be made. 
    First, DAFL is most effective when both $J$ and $\mathcal{I}$ are high. In D2, D4, D5, and D6, DAFL achieves a higher minority-class F1-score than the baselines while requiring less training time. 
    This improvement arises because reducing $J$ and $\mathcal I$ improves minority-class performance and accelerates convergence, respectively.
    Second, low $J$ and high $\mathcal I$ leave little scope to improve the minority-class F1-score, hence, DAFL mainly reduces the convergence time. In contrast, for low $\mathcal I$ and high $J$, DAFL improves the minority-class F1-score at the cost of increased training time. This explains the performance observed in D1 and D3. In D6, since $J$ is moderate, DAFL provides a minor improvement in the minority-class F1-score at the cost of reduced accuracy and macro F1-score compared to other baselines. This is because the other baselines improve minority-class performance less than that of DAFL, resulting in a lower minority-class F1-score but higher accuracy and macro F1-score than DAFL as plotted in Fig.~\ref{fig:I_0_089_J_0_1309_minority_f1score}-\ref{fig:I_0_089_J_0_1309_mean_f1score}.
    Third, since LB requires more augmentation than DAFL, when minority classes contain very few samples, repeated rotation and flipping generate duplicate images. This results in overfitting, leading to a poor minority-class F1-score. 
    Consequently, DAFL outperforms LB in minority-class performance while using fewer augmented samples in five of the six datasets. The exception is D3, where excessive augmentation to reduce $\mathcal{I}$ causes overfitting.
    Finally, Table~\ref{tab:D2comparison} compares DAFL with FedProx and Focal Loss on D2, the most imbalanced setting, to analyze how effectively DAFL addresses severe class imbalance. 
    DAFL achieves both the highest minority-class F1-score and the shortest training time, demonstrating its effectiveness under severe global class imbalance and label proportion imbalance.

\begin{appendix}
    \subsection{Proof of Lemma \ref{lem:gradient_decomposition}} \label{proof:gradient_decomposition}
    For any $k, \mathbf w$, we write $\nabla\! F_k(\mathbf w) \! = \! \nabla\! F_k(\mathbf w)\! - \!\nabla\! F(\mathbf w) \! +\! \nabla\! F(\mathbf w),$ wherein $\nabla\! F_k(\mathbf w)\! - \!\nabla\! F(\mathbf w)$ is as given in \eqref{eqn:hetero-decomp}. Then
    \begin{equation} \label{eq:before-uk-vk}
        \|\nabla\! F_k(\mathbf w)\|^2\! \le \!2\|\nabla\! F_k(\mathbf w)\! - \!\nabla\! F(\mathbf w)\|^2\! +\! 2\|\nabla\! F(\mathbf w)\|^2\!.
    \end{equation}
    Subsequently, applying Cauchy-Schwarz inequality  and using \eqref{eqn:delta_k2} and \textbf{A3}, we bound $u_k(\mathbf w)$ in \eqref{eqn:hetero-decomp} and, using \textbf{A5}, we bound $v_k(\mathbf w)$ in \eqref{eqn:hetero-decomp} as,
        \begin{IEEEeqnarray}{rcl}
            \|{u_k(\mathbf w)}\|^2 \le CG^2\delta_k^2,\qquad
            \|{v_k(\mathbf w)}\|^2 \le \Gamma_k^2 \qquad \forall \mathbf w \label{eq:uk-vk-bound}
        \end{IEEEeqnarray}    
    \noindent respectively. Hence, we obtain \eqref{eqn:lemma_gradient_bound} in Lemma \ref{lem:gradient_decomposition} on substituting \eqref{eq:uk-vk-bound} into \eqref{eq:before-uk-vk}. This completes the proof.
    
    \subsection{Proof of Theorem \ref{thm:convergence}}\label{proof:thm_convergence}
    
    From \textbf{A1} and \eqref{eqn:E_local_update}, we get
        \begin{IEEEeqnarray}{rcl}\label{eqn:smoothness_inequality_applied}
        F(\mathbf{w}^{(t+1)})
        \le
        F(\mathbf{w}^{(t)})
       & - &
        \eta
        \langle
        \nabla F(\mathbf{w}^{(t)}),
        \mathbf{g}^{(t)}
        \rangle +
        \frac{L\eta^2}{2}
        \|\mathbf{g}^{(t)}\|^2. \nonumber
    \end{IEEEeqnarray}
    
    \noindent Taking, $\mathbb E_t[\cdot]\!:= \!\mathbb E[\cdot\! \mid\! \mathbf w^{(t)}],$ 
    gives \eqref{eqn:conditional_expectation_on_smoothness_inequality}. To proceed along the lines of the standard proof for convergence, we need to upper bound the inner product term $\langle
        \nabla F(\mathbf{w}^{(t)}),
        \mathbb{E}_t[\mathbf{g}^{(t)}]
        \rangle $
        as a function of $\|\nabla F(\mathbf{w}^{(t)})\|^2$ along with some constants. Similarly, we need to upper bound $\mathbb{E}_t[\|\mathbf{g}^{(t)}\|^2]$ as a function of $\|\nabla F(\mathbf{w}^{(t)})\|^2$ along with some constants. Subsequently, we can rearrange the terms in \eqref{eqn:conditional_expectation_on_smoothness_inequality} for each $t$, and then apply the telescopic sum from $t = 0$ to $t = T-1$, to obtain an upper bound on $\frac{1}{T}\! \sum_{t=0}^{T-1}\!\mathbb{E}[ \|\nabla F(\mathbf{w}^{(t)})\|^2 ]$. To achieve this, we use the connection between the local gradient $\nabla F_{k}(\mathbf{w}^{(t)})$ and the global gradient $\nabla F(\mathbf{w}^{(t)})$ as given in Lemma \ref{lem:gradient_decomposition}.

The term $\mathbb{E}_t[\mathbf{g}^{(t)}]$ is a noisy version of $\nabla F_{k}(\mathbf{w}^{(t)})$ owing to the drift contributed by local updates at each client. This noise is captured by $\mathbf{z}^{(t)} \!=\! {\sum_{i=0}^{E-1} \sum_k p_k(\nabla F_k(\mathbf w_k^{(t,i)})\! -\! \nabla F_k(\mathbf w^{(t)}))}$ as 
\begin{equation}
    \mathbb E_t[\mathbf g^{(t)}] = E \nabla F(\mathbf w^{(t)}) + \mathbb E_t[\mathbf z^{(t)}]
         \label{eqn:firstmoment} 
\end{equation}
presented in  \eqref{eqn:firstmoment}. When we apply $\mathbb{E}_t[\mathbf{g}^{(t)}]$ in the inner product term of \eqref{eqn:conditional_expectation_on_smoothness_inequality}, we would need to obtain a bound on $\|\mathbb{E}_t[\mathbf{z}^{(t)}]\|^2$. To provide that, we present the bound in \eqref{eq:cor1}. With this, one can bound the inner product term as a function of $\|\nabla F(\mathbf{w}^{(t)})\|^2$ and some other constants.

The next task is to upper bound $\mathbb{E}_t[\|\mathbf{g}^{(t)}\|^2]$ as a function of $\|\nabla F(\mathbf{w}^{(t)})\|^2$ along with some constants. By definition, $\mathbf{g}^{(t)}$, which is a noisy version of $\nabla F(\mathbf{w}^{(t)})$, is a sum of three terms.    Using \eqref{eqn:aggregated_stochastic_gradient}, \textbf{A2}, and applying the identity $\|{a+b+c}\|^2 \!\le\! 3\| a\|^2\!+\!3\| b\|^2\!+\!3\| c\|^2$, we get,
\begin{equation*}
    \|{\mathbf g^{(t)}}\|^2 \le 3E^2\|{\nabla F(\mathbf w^{(t)})}\|^2 + 3\|\mathbf z^{(t)}\|^2 + 3\|\mathbf n^{(t)}\|^2,
         \label{eq:lem4-b}
\end{equation*}
where $\mathbf{n}^{(t)} = {\textstyle\sum_{i=0}^{E-1}(\textstyle\sum_k p_k(\mathbf g_k^{(t,i)} - \nabla F_k(\mathbf w_k^{(t,i)})))}$. After taking $\mathbb E_t[\cdot]$ on both sides, we bound the middle term by \eqref{eq:cor1}. Subsequently using \textbf{A4} and multiplying both sides by $\tfrac{L\eta^2}2$, we obtain \eqref{eq:step6-d}.
    Substituting \eqref{eqn:firstmoment}, \eqref{eq:cor1} and \eqref{eq:step6-d} into \eqref{eqn:conditional_expectation_on_smoothness_inequality} gives \eqref{eq:step6-e} where {\small $A := L^2\eta^3E^2(E-1)^2, B := L\eta A$}.

    Under {\small $\eta LE\le\tfrac1{10}$, $B=L\eta A\le\tfrac1{10}A$}, we get {\small $32A+24B\le34.4A\le35A$} and {\small $4A+3B\le4.3A\le5A$}. Substituting these bounds into \eqref{eq:step6-e}, we get 
    \eqref{eq:descent-final}.
    On summing over $t=0,\ldots,T-1$, rearranging the terms after using $F(\mathbf{w}^{(T)})\ge F(\mathbf{w}^*)$ and dividing by $\frac{\eta E}{2}T$, we get \eqref{eq:telescope-final}. Finally, assuming $T\ge100L^2E$ and choosing $\eta = \frac{1}{\sqrt{ET}}$, we get \eqref{eqn:thm_convergence_rate}.
\end{appendix}
\bibliographystyle{IEEEtran}
\bibliography{bibliography}
\end{document}